\documentclass{article}
\usepackage{fullpage}
\title{Variance-Reduced Off-Policy Memory-Efficient Policy Search}
\author{
  Daoming Lyu\\
  Auburn University\\
  \texttt{daoming.lyu@auburn.edu} \\
  \and
  Qi Qi\\
  University of Iowa\\
  \texttt{qi-qi@uiowa.edu} \\
  \and
  Mohammad Ghavamzadeh\\
  Facebook AI Research\\
  \texttt{mohammad.ghavamzadeh@inria.fr} \\
  \and
  Hengshuai Yao\\
  Huawei Inc.\\
  \texttt{hengshuai.yao@huawei.com} \\
   \and
  Tianbao Yang\\
  University of Iowa\\
  \texttt{tianbao-yang@uiowa.edu} \\
  \and
  Bo Liu\\
  Auburn University\\
  \texttt{boliu@auburn.edu} \\
}

\usepackage{algorithm}
\usepackage{algorithmic}
\usepackage{epsf}

\usepackage{amssymb}
\usepackage{wrapfig}
\usepackage{mathtools}
\newcommand{\E}{\mathbb{E}}

\usepackage{mathtools}
\usepackage{mathrsfs}
\mathtoolsset{showonlyrefs}
\newcounter{RomanNumber}

\usepackage{tikz}
\newcommand*\circled[1]{\tikz[baseline=(char.base)]{
            \node[shape=circle,draw,inner sep=1pt] (char) {#1};}}
            
\usetikzlibrary{positioning,decorations.pathreplacing,fit}

\usepackage{graphicx}
\usepackage{geometry}
\usepackage{amsmath}
\usepackage{amsfonts}
\usepackage{xcolor}
\definecolor{darkblue}{HTML}{000080}
\usepackage{multirow}
\usepackage{url}
\usepackage[colorlinks=true,linkcolor=black,citecolor=darkblue]{hyperref}
\usepackage{mathtools}
\usepackage{mathrsfs}
\mathtoolsset{showonlyrefs}
\usepackage{hyperref}
\usepackage{natbib}
\setcitestyle{authoryear} 
\setcitestyle{square}
\setcitestyle{semicolon} 
\usepackage{amsthm}

\usepackage{tcolorbox}
\newcounter{thm_counter}
\newcounter{lem_counter}

\newcounter{ass_counter}

\newcounter{rmk_counter}
\newtheorem{theorem}[thm_counter]{Theorem}
\newtheorem{lemma}[lem_counter]{Lemma}

\newtheorem{assumption}[ass_counter]{Assumption}

\newtheorem{remark}[rmk_counter]{Remark}
\usepackage[finalold]{trackchanges}

\usepackage{microtype}
\usepackage{graphicx}
\usepackage[labelformat=simple]{subcaption}

\usepackage{booktabs} 

\usepackage{hyperref}
\hypersetup{draft}

\begin{document}
  
\maketitle
\setcounter{footnote}{0}

\begin{abstract}
Off-policy policy optimization is a challenging problem in reinforcement learning (RL). The algorithms designed for this problem often suffer from high variance in their estimators, which results in poor sample efficiency, and have issues with convergence. A few variance-reduced on-policy policy gradient algorithms have been recently proposed that use methods from stochastic optimization to reduce the variance of the gradient estimate in the REINFORCE algorithm. However, these algorithms are not designed for the off-policy setting and are memory-inefficient, since they need to collect and store a large ``reference'' batch of samples from time to time. To achieve variance-reduced off-policy-stable policy optimization, we propose an algorithm family that is memory-efficient, stochastically variance-reduced, and capable of learning from off-policy samples. Empirical studies validate the effectiveness of the proposed approaches.
\end{abstract}

\section{Introduction}
\textit{Off-policy control and policy search} is a ubiquitous problem in real-world applications wherein the goal of the agent is to learn a near-optimal policy $\pi$ (that is, close to the optimal policy $\pi^*$) from samples collected via a (non-optimal) behavior policy $\mu$. Off-policy policy search is important because it can learn from previously collected data generated by non-optimal policies, such as from demonstrations (LfD)~\citep{dqn-lfd:2018}, from experience replay~\citep{mnih2015human}, or from executing an exploratory (even randomized) behavior policy. It also enables learning multiple tasks in parallel through a single sensory interaction with the environment~\citep{sutton2011horde}. However, research into efficient off-policy policy search has encountered two major challenges: off-policy stability and high variance. There is work in each direction, e.g., addressing off-policy stability~\citep{imani2018off,geoffpac} via emphatic weighting~\citep{sutton2016emphatic,hallak2016generalized}, and reducing the high variance caused by ``curse of horizon''~\citep{liu2018breaking,xie2019margin}, but little that addresses both challenges at the same time.

\emph{Stochastic variance reduction} has recently emerged as a strong alternative to stochastic gradient descent (SGD) in finding first-order critical points in non-convex optimization. The key idea is to replace the stochastic gradient (used by vanilla SGD techniques) with a ``\textit{semi-stochastic}'' gradient for objective functions with a finite-sum structure. A semi-stochastic gradient combines the stochastic
gradient in the current iterate with a snapshot of an earlier iterate, called the {\em reference iterate}. This line of research includes methods such as SVRG~\citep{svrg,zhang2013linear}, SAGA~\citep{saga}, SARAH~\citep{sarah}, and SPIDER~\citep{spider}.
A common feature of these techniques is storing a ``reference'' sample set in memory to estimate the gradient at a ``checkpoint,'' and then using it in updates across different training epochs. The reference set is usually very large---$O(n)$ in SVRG, for example, where $n$ is the size of the training data. This is a significant obstacle limiting the application of these variance-reduction techniques in deep learning. 
There has been a recent surge in research applying these ``semi-stochastic'' 
gradient methods to policy search to help reduce variance~\citep{svrpg,xu2019improved,sarahpg}, for example.
However, there are two drawbacks with these algorithms. The first is that a large ``reference" sample set must be stored, which is memory costly. The second is that these algorithms lack off-policy guarantees because they adopt the REINFORCE ~\citep{williams1992simple} algorithm as the policy search subroutine, which is only on-policy stable.  

In this paper, we aim to address the memory-efficiency and off-policy stability issues of existing stochastic variance-reduced policy search methods, we propose a novel \textit{variance-reduced off-policy} policy search algorithm that is both {\em convergent} and {\em memory efficient}.
To this end, we introduce novel ingredients, i.e., STOchastic
Recursive Momentum (STORM)~\citep{storm}, Actor-Critic with Emphatic weightings (ACE)~\citep{imani2018off}, and Generalized Off-Policy Actor-Critic (GeoffPAC)~\citep{geoffpac}.  
Combining the novel components of ACE/GeoffPAC and STORM offers a number of advantages. 
First, \textit{ACE and GeoffPAC are off-policy stable.}
We choose ACE and GeoffPAC especially they are the only two stable off-policy policy gradient approaches to the best of our knowledge.\footnote{\citet{degris2012off} proposed the OffPAC algorithm with certain theoretical incoherence, which was then fixed in~\citep{imani2018off}.} 
Second, \textit{STORM is memory-efficient}.
STORM is so far the only stochastic variance-reduced algorithm that need not revisit a ``fixed'' batch of samples. 
Based on these key ingredients, we propose the \textbf{V}ariance-reduced \textbf{O}ff-policy \textbf{M}emory-efficient \textbf{P}olicy 
\textbf{S}earch (VOMPS) algorithm and the  \textbf{A}ctor-\textbf{C}ritic with \textbf{E}mphatic weighting and \textbf{STO}chastic \textbf{R}ecursive \textbf{M}omentum (ACE-STORM) algorithm,
with two primary contributions: 
\textbf{(1)} VOMPS and ACE-STORM are both off-policy stable. Previous approaches are on-policy by nature (by adopting REINFORCE as their policy search component), and thus cannot be applied to the off-policy setting. To the best of our knowledge, VOMPS and ACE-STORM are the \textit{first} off-policy variance-reduced policy gradient methods. \textbf{(2)} The two algorithms are memory-efficient. Unlike previous approaches that must store a large number of samples for reference-point computation, our new algorithms do not need a reference sample set and are thus memory efficient.

Here is a roadmap to the rest of the paper. Sec.~\ref{sec:prelim} of this paper follows by introducing background on stochastic variance reduction. 
Sec.~\ref{sec:alg} develops the algorithms and conducts sample-complexity analysis. An empirical study is conducted in Sec.~\ref{sec:exp}. Then Sec.~\ref{sec:related} contains more detailed related work and Sec.~\ref{sec:conclusion} concludes the paper.

\section{Preliminaries}
\label{sec:prelim}
In this section, we provide a brief overview of variance-reduction techniques in non-convex stochastic gradient descent and off-policy policy search algorithms in reinforcement learning. In particular, we describe the two main building blocks of our work: STORM ~\citep{storm} and GeoffPAC ~\citep{geoffpac}.
\subsection{Stochastic Variance Reduction and STORM}
\label{prelim:svr}
STORM ~\citep{storm} is a state-of-the-art stochastic variance-reduction algorithm that avoids the reference sample set storage problem.  
The stochastic optimization problem is of the form $J(x)=\min_{x \in \mathbb{R}^d} \mathbb{E} \big[f(x, \xi)\big]$, 
where the function $J:\mathbb{R}^d \rightarrow  \mathbb{R}$ can be thought of as the training loss of a machine learning model, and $f(x, \xi)$ represents the loss of a sample $\xi$ for the parameter $x \in \mathbb{R}^d$. In this setting, SGD produces a sequence of iterates $x_1,\dots,x_T$ using the recursion $x_{t+1} = x_t - \eta_t \nabla f(x_t,\xi_t)$, where $f(\cdot,\xi_1),\dots,f(\cdot,\xi_T)$ are i.i.d.~samples and $\eta_1,\dots\eta_T\in \mathbb{R}$ is a sequence of stepsizes. 
STORM replaces the gradient in the SGD's update with %
\begin{align}
g_t = \underbrace{(1- \alpha_t) g_{t-1} + \alpha_t \nabla f(x_{t},\xi_t)}_{\circled{1}} + \underbrace{ (1-\alpha_t)(\nabla f(x_t,\xi_t) - \nabla f(x_{t-1}, \xi_t))}_{\circled{2}}~,
\label{eq:STORM-Grad}
\end{align}
where $\alpha_t \in [0,1]$ is the momentum parameter, $\circled{1}$ is the update rule of vanilla SGD with momentum, and $\circled{2}$ is an additional term introduced to reduce variance.
STORM achieves the so-far optimal convergence rate of $O(1/\epsilon^{3/2})$ to find a $\epsilon$-stationary point---$||\nabla J(x)||^2 \leq \epsilon$. (We report the convergence rates of several variance reduction algorithms in Appendix~\ref{sec:compare-svr} as well.)
Thus STORM achieves variance reduction using a version of the momentum term,
and does not use the estimated gradient at a checkpoint in its update. It alleviates the need to store a large reference sample set and therefore is memory-efficient. 
\subsection{Reinforcement Learning and Off-Policy Policy Search}
\label{prelim:emp}
In RL, the agent's interaction with the environment is often modeled as a Markov Decision Process (MDP), which is a tuple $({\mathcal{S},\mathcal{A},p,r,\gamma})$, where $\mathcal{S}$ and $\mathcal{A}$ are the state and action sets, the transition kernel $p(s'|s,a)$ specifies the probability of transition from state $s\in\mathcal{S}$ to state $s'\in\mathcal{S}$ by taking action $a\in\mathcal{A}$, $r(s,a):\mathcal{S}\times\mathcal{A}\to\mathbb{R}$ is a bounded reward function, and $0\leq\gamma<1$ is a discount factor. 
Given a (stochastic) policy $\pi: \mathcal{S} \times \mathcal{A} \rightarrow [0, 1]$, $V_\pi:\mathcal{S}\rightarrow\mathbb R$ is the associated state value function,
$Q_\pi:\mathcal{S} \times \mathcal{A} \rightarrow\mathbb R$ the state-action value function,
and $P_\pi$ the transition kernel, $P_\pi(s'|s) = \sum_a \pi(a | s)p(s^\prime|s, a)$. 
In policy gradient methods, $\pi$ is often approximated in a parametric form $\pi_\theta$ which is differentiable with respect to its parameter~$\theta$.

In the off-policy setting, an agent aims to learn a target policy $\pi$ from samples generated by a behavior policy $\mu$. We assume that the Markov chains induced by policies $\pi$ and $\mu$ are ergodic, and denote by $d_\pi$ and $d_\mu$ their unique stationary distributions. The stationary distribution matrices are $D_\pi := {\rm Diag} (d_\pi)$ and $D_\mu := {\rm Diag} (d_\mu)$. The standard coverage assumption for $\pi$ and $\mu$ is used, $\forall (s, a), \pi(a | s) > 0$ implies $\mu(a | s) > 0$~\citep{sutton2018reinforcement}.
With this assumption, the non-trivial importance sampling ratio is well defined, $\rho(s, a) := \frac{\pi(a | s)}{\mu(a | s)}$. For simplicity, we use $\rho_t := \rho(s_t, a_t)$ for the importance sampling ratio at time $t$.
\textit{Distribution mismatch} between the stationary distributions of the behavior and the target policies is the primary challenge in off-policy learning. To correct this mismatch, ~\citet{sutton2016emphatic} introduced {\em emphatic weighting}, where for a given state $s$ an emphatic weight $M(s)$ is computed to offset the state-wise distribution mismatch. This technique has recently been widely used for off-policy value function estimation~
\citep{sutton2016emphatic,hallak2017consistent} and policy optimization~\citep{imani2018off,geoffpac}.

In policy gradient literature, different objectives have been used. In the on-policy continuing task setting, the goal is often to optimize the \textit{alternative life objective} $J_\pi = \sum_s d_\pi(s) i(s) V_\pi(s)$~\citep{silver2015reinforcement}, which is equivalent to optimizing the average reward objective~\citep{puterman2014markov}, when $\gamma=1$ and interest function $i(s)=1,\;\forall s\in\mathcal S$. 
On the other hand, in the off-policy continuing task setting where $d_\pi$ is difficult to achieve due to that the samples are collected from the behavior policy ~\citep{imani2018off}, it is more practical to resort to the \textit{excursion objective}~\citep{imani2018off}---that is,~$J_\mu := \sum_s d_\mu(s)i(s)V_\pi(s)$ instead of $J_\pi$, where $d_\pi$ (in $J_\pi$) is replaced by $d_\mu$ (in $J_\mu$). However, the excursion objective does not correctly represent the state-wise weighting of the target policy $\pi$'s performance~\citep{gelada2019off}. 
To address this,~\citet{geoffpac} introduced the \textit{counterfactual objective}, $J_{\hat{\gamma}}$, to unify $J_\mu$ and $J_\pi$ in the continuing RL setting:
\begin{align}
J_{\hat{\gamma}} := \sum_s d_{\hat{\gamma}}(s) \hat{i}(s) V_\pi(s),
\label{eq:obj_general}
\end{align}
where $\hat{\gamma} \in [0, 1]$ is a constant, and $d_{\hat{\gamma}}$ is the stationary distribution of the Markov chain with transition matrix ${{\rm{P}}_{\hat \gamma }} = \hat \gamma {{\rm{P}}_\pi } + (1 - \hat \gamma )\mathbf{1}{d^ \top _\mu}$. $d_{\hat{\gamma}} = (1 - \hat{\gamma})(\mathbf{I} - \hat{\gamma} P_\pi^\top)^{-1} d_\mu$ ($\hat{\gamma}<1$) and $d_{\hat{\gamma}} =d_{\pi}$ ($\hat{\gamma}=1$), and
$\hat{i}$ is a user-defined extrinsic interest function. In these equations, $\mathbf{I}$ and $\mathbf{1}$ are the identity matrix and all-one column vector. 
\citet{geoffpac} argue that $J_{\hat{\gamma}}$ is potentially a better objective for off-policy control, for the following reasons: \textbf{1)} $J_{\hat{\gamma}}$ is more general than $J_{\pi}$ and $J_\mu$, since $J_\pi$ and $J_\mu$ can be recovered from $J_{\hat{\gamma}}$ for $\hat{\gamma} = 1$ and $\hat{\gamma} = 0$, respectively. This is because for $\hat{\gamma} = 1$ and $\hat{\gamma} = 0$, we have $d_{\hat{\gamma}} = d_{\pi}$ and $d_{\hat{\gamma}} = d_{\mu}$~\citep{gelada2019off}. An intermediate $\hat{\gamma}$ tweaks the stationary distribution towards that of the target policy and makes the objective closer to the original alternative life objective. 
\textbf{2)} $J_{\hat{\gamma}}$ is more suitable than $J_\mu$ for the off-policy setting, as it better reflects state-wise weighting of $\pi$'s performance and typically leads to a better empirical performance according to the observation of~\citep{geoffpac}. 
The Generalized Off-Policy Actor-Critic (GeoffPAC) algorithm is a state-of-the-art approach that optimizes $J_{\hat{\gamma}}$. A key component of the GeoffPAC algorithm is the \textit{emphatic weight update component}, which is discussed in detail in the Appendix.
As noted above, when $\hat \gamma = 0$, the stationary distribution $d_{\hat{\gamma}}$ reduces to $d_{\mu}$, and correspondingly the objective of GeoffPAC ($J_{\hat{\gamma}}$) reduces to the that of Actor-Critic with Emphatic-weighting (ACE) algorithm ($J_\mu$)~\citep{imani2018off}.
\section{Algorithm Design and Analysis}
\label{sec:alg}
\subsection{VOMPS Algorithm Design}

We consider off-policy policy optimization of infinite-horizon discounted MDP problems, which is identical to the problem setting of ACE and GeoffPAC. Our new algorithm the \textbf{V}ariance-reduced \textbf{O}ff-policy \textbf{M}emory-efficient \textbf{P}olicy \textbf{S}earch (VOMPS) is presented in Algorithm~\ref{alg:vomps}. For simplicity, the subscript of $\pi$ is omitted for $V_{\pi}$ and $Q_{\pi}$. 
We denote the state value function as $V(s; \nu)$ and the policy function as $\pi(a|s;\theta)$ with $\nu$ and $\theta$ being their parameters. A parametric approximation of the density ratio function, $C(s;\psi)$ is introduced to reweight online updates to the value
function in order to avoid divergence issues in
off-policy learning~\citep{gelada2019off,geoffpac}.

VOMPS is an off-policy actor-critic method that uses emphatic weighting based policy gradient for off-policy stability guarantee, and stochastic recursive momentum for memory-efficient variance reduction.
Algorithm~\ref{alg:vomps} is illustrated below in the order of updating the critic, the density ratio, the emphatic weights, and the actor. The hyperparameters in the algorithm are identified in the Appendix.
\textit{{First}}, the \textit{critic update} is conducted using a Temporal Difference (TD) method:
\begin{align}
\delta_{t} = r_{t} + \gamma V(s_{t+1}; \nu_{t}) - V(s_{t}; \nu_{t}), \qquad \nu_{t+1} = \nu_{t} + \alpha_{\nu} \delta_{t} \nabla_{\nu}V(s_{t}; \nu_{t})~,
\label{eq:critic}
\end{align}
where $\delta_{t}$ is the TD error at $t$-th timestep, and $\alpha_{\nu}$ is the stepsize.
In fact, the critic is not limited to the TD method and can be replaced by other approaches in order to improve the value function estimation.
\textit{{Second}}, the \textit{density ratio update} is performed:
\begin{align}
    \psi_{t+1} = \psi_{t} + \alpha_{\psi} \big(\hat{\gamma} \rho_{t} C(s_{t}; \psi_{t}) + (1 -\hat{\gamma}) - C(s_{t+1}; \psi_{t}) \big)\nabla_{\psi}C(s_{t+1}; \psi_{t}) ~,
    \label{eq:density}
\end{align}
where $\alpha_{\psi}$ is the stepsize.
\textit{{Third}}, we conduct the \textit{emphatic weights update} of $M_t^{(1)}, M_t^{(2)}$ that are used to correct the impact of the distribution discrepancy between $\pi$ and $\mu$. For the \textit{counterfactual objective} of $J_{\hat{\gamma}}$, the policy gradient $\nabla J_{\hat{\gamma}}$ is computed as follows: 
\begin{align}
    \nabla J_{\hat{\gamma}} &= \sum_s d_{\hat{\gamma}}(s)\hat{i}(s)\nabla V(s) + \sum_s \nabla d_{\hat{\gamma}}(s)\hat{i}(s)V(s) \\
    & = \sum_s d_\mu(s) C(s) \hat{i}(s) \sum_a Q(s, a)\nabla \pi(a|s) + \sum_s d_\mu(s) \nabla C(s)\hat{i}(s)V(s) \\
    &= \mathbb{E}_\mu \big[M_t^{(1)} \rho_t \delta_t \nabla \log \pi(a_t|s_t) + \hat{\gamma} M_t^{(2)} V(s_t) \hat{i}(s_t) \big] ~.
    \label{eq:grad_general}
\end{align}
Specifically, $M_t^{(1)}$ (resp. $M_t^{(2)}$) is used to adjust the weighting of $\rho_t \delta_t \nabla \log \pi(a_t|s_t)$ (resp. $\hat{\gamma} V(s_t) \hat{i}(s_t)$) caused by the discrepancy between $d_\pi$ and $d_\mu$.  Details of the update law is shown in Appendix~\ref{sec:app:emphatic}, which is adopted from the emphatic weight update component proposed in~\citep{geoffpac}.
Let $Z_t :=M^{(1)}_t \rho_t \delta_t \nabla \log \pi(a_t|s_t) + \hat{\gamma} M_t^{(2)} V(s_t) \hat{i}(s_t) $ be an estimate of the policy gradient at time $t$, then according to~\citep{geoffpac}, we have $\mathbb{E}_\mu [Z_t]= \nabla J_{\hat{\gamma}}$. That is, our estimation of the policy gradient is unbiased, as shown in the last equality of Eq.~\eqref{eq:grad_general}. 
\textit{{Fourth}}, the \textit{actor update} via policy gradient is conducted.
Instead of using the vanilla actor update in~\citep{geoffpac}, we introduce the STORM~\citep{storm} technique to reduce the variance in the gradient estimates. According to the technique used in STORM~\citep{storm}, both $Z_t(a_t,s_t;\theta_{t-1})$ and $Z_t(a_t,s_t;\theta_{t})$ need to be calculated as follows:
\begin{align}
    Z_t(a_t,s_t;\theta_{t}) &= M^{(1)}_t \rho_{t} \delta_{t} \nabla_{\theta} \log \pi(a_{t} | s_{t};\theta_{t}) + \hat{\gamma} M^{(2)}_t V(s_{t};\nu) \hat{i}(s_{t}) ~,
            \label{eq:zt_now}\\
    Z_t(a_t,s_t;\theta_{t-1}) &= M^{(1)}_t \rho_{t} \delta_{t} \nabla_{\theta} \log \pi(a_{t} | s_{t};\theta_{t-1}) + \hat{\gamma} M^{(2)}_t V(s_{t};\nu) \hat{i}(s_{t}) ~.
            \label{eq:zt_prev}
\end{align}
The actor's update law is formulated as $\theta_{t+1} = \theta_t + \eta_t g_{t}$, where the two key ingredients are the \textit{stochastic recursive momentum} update term $g_{t}$ and the \textit{adaptive stepsize} $\eta_t$. 
The update term $g_t$ is computed as
\begin{align}
    g_{t} &= Z_t(a_t,s_t;\theta_{t}) + (1 - \alpha_t)\big(g_{t-1} - Z_t(a_t,s_t;\theta_{t-1})\big)~,  \label{eq:gt_general} 
\end{align}
and the adaptive stepsizes $\eta_t$ and $\alpha_t $ are computed as follows, with $k$, $w$, and $\beta$ inherited from STORM,
\begin{align}
G_{t} = \|Z_t(a_t,s_t;\theta_{t})\|, \quad \eta_t = k/(w + \sum^{t}_{i=1} {G_t^2})^{1/3}~,~\quad
\alpha_{t} = \beta \eta^2_{t-1}.
\label{eq:stepsize}
\end{align}
It should be noted that $Z_t(a_t,s_t;\theta_{t})$ is used in Eq.~\eqref{eq:gt_general} \&~\eqref{eq:stepsize}, while $Z_t(a_t,s_t;\theta_{t-1})$ is used in Eq.~\eqref{eq:gt_general}.


\begin{algorithm}[t]
\caption{ Variance-reduced Off-policy Memory-efficient Policy Search (VOMPS)}
\label{alg:vomps}
$V(s; \nu)$: state value function parameterized by $\nu$\;\\
$C(s; \psi)$: density ratio estimation parameterized by $\psi$\;\\
$\pi(a|s; \theta)$: policy function parameterized by $\theta$\; 
\begin{algorithmic}[1] 
\STATE \textbf{Input}:  Parameters $\;\psi$, $\nu$, $\theta$;

    \FOR{timestep $t=0$ to $T$}
        \STATE Sample a transition $s_t$, $a_t$, $r_t$, $s_{t+1}$ according to behavior policy $\mu$.
        \STATE \textbf{Critic update} according to Eq.~\eqref{eq:critic}.
        \STATE \textbf{Density ratio update} according to Eq.~\eqref{eq:density}.
        \STATE \textbf{Emphatic weights update}: update $M^{(1)}_t, M^{(2)}_t, Z_t(a_t,s_t;\theta_{t})$ as in Figure~\ref{fig:geoffpac} in the Appendix.
        \STATE Compute actor stepsize as in Eq.~\eqref{eq:stepsize}, $Z_t(a_t,s_t;\theta_{t-1})$ as in Eq.~\eqref{eq:zt_prev} and $g_{t}$ as in Eq.~\eqref{eq:gt_general}.
        \STATE \textbf{Actor update} as $    \theta_{t+1} = \theta_t + \eta_t g_{t}$.
    \ENDFOR
\STATE \textbf{Output I}: Parameters $\psi_{T+1}$, $\nu_{T+1}$, $\theta_{T+1}$.
\STATE \textbf{Output II}:Parameters $\psi_{T+1}$, $\nu_{T+1}$, $\theta_\tau$, where $\tau$ is sampled with a probability of $p(\tau  = t)\propto \frac{1}{\eta_t^2}$.
\end{algorithmic}
\end{algorithm}

\subsection{Theoretical Analysis}
In this section, we present a theoretical analysis of the VOMPS algorithm. To start, we first present the assumptions used in the study.
\begin{assumption}[\textbf{Bounded Gradient}]
\label{assump:grad}
\citep{svrpg,sarahpg}
Let $\pi_{\theta}(a|s)$ be the agent's policy at state $s$. There exist constants $W,U>0$ such that the log-density of the policy function satisfies:
$
\|\nabla_{\theta}\log \pi_{\theta}(a|s)\|_2\leq W,\quad \big\|\nabla_{\theta}^2\log \pi_{\theta}(a|s)\big\|_2\leq U,
$
for $\forall$ $a\in\mathcal{A}$ and $s\in\mathcal{S}$, and $||\cdot||_2$ is the $\ell_2$ norm.
\end{assumption}

\begin{assumption}[\textbf{Lipschitz continuity and Bounded Variance}]\citep{xu2019improved,sarahpg,storm}
The estimation of policy gradient $Z(\theta)$ is bounded, Lipschitz continuous, and has a bounded variance, i.e.,  
there exist constants $L, G, \sigma$ such that
$
\|Z(\theta_1) - Z(\theta_2)\|_2 \leq L\|\theta_1-\theta_2\|_2
$
for $\forall$ $\theta_1,\theta_2\in \mathbb{R}^d$, and $\|Z(\theta)\|_2 \leq G, \mathbb{E}[\|Z(\theta) -\nabla J_{\hat{\gamma}}(\theta)\|_2^2] \leq {\sigma}^2$ for $\forall$ $\theta\in\mathbb{R}^d$.
\end{assumption}

We now present our main theoretical result, the convergence analysis of Algorithm~\ref{alg:vomps}.
\begin{theorem}
Under the above assumptions, for any $b>0$, let $k=\frac{b \sigma^\frac{2}{3}}{L}$, $\beta=28L^2 + \sigma^2/(7 L k^3) = L^2(28 + 1/(7 b^3))$, and $w=\max\left((4Lk)^3, 2\sigma^2, \left(\tfrac{\beta k}{4 L}\right)^3\right) = \sigma^2\max\left((4 b)^3, 2, (28b+\frac{1}{7b^2})^3/64\right)$. Then, the output of Algorithm~\ref{alg:vomps} satisfies
\begin{align}
    \mathbb{E} \left[\|\nabla J_{\hat{\gamma}}(\hat{\theta})\|^2\right]
= \mathbb{E} \left[\frac{1}{T}\sum_{t=1}^T \|\nabla J_{\hat{\gamma}}(\theta_t)\|^2\right]
\leq \frac{40\Delta_{\Phi}}{T^{2/3}} + \frac{2\beta^2\sigma^2}{L^2T^{2/3}},
\end{align}
where $\Delta_{\Phi} \leq \Delta_{J_{\hat{\gamma}}} + \frac{\|\epsilon_0\|^2}{32\eta_0L^2},\Delta_{J_{\hat{\gamma}}} =   J_{\hat{\gamma}}(\theta^*)-J_{\hat{\gamma}}(\theta) , \forall \theta\in R^d$, and $\theta^\star$ is the maximizer of $J_{\hat{\gamma}}$.
\label{thm:vomps}
\end{theorem}
Theorem~\ref{thm:vomps} indicates that VOMPS requires $O(1/\epsilon^{3/2})$ samples to find an $\epsilon$-stationary point. As VOMPS is developed based on GeoffPAC and STORM, the proof of convergence rates of VOMPS is similar to STORM. 
 However, the proof is not a trivial extension by merely instantiating the objective function to $J_{\hat{\gamma}}$ in the RL settings.
  If we just apply the original analysis of STORM, we can only achieve $O(\frac{\log(1/\epsilon)}{\epsilon^{2/3}})$. 
In \cite{ yuan2020stochastic}, it improves the sample complexity to $O(1/\epsilon^{2/3})$ by introducing the large mini-batch $O(1/\sqrt{\epsilon})$ and using extremely small stepsize $o(\epsilon)$. 
Nevertheless the improved sample complexity, 
the introduced $O(1/\sqrt{\epsilon})$ mini-batch size will lead to memory inefficiency, and the $O(\epsilon)$ stepsize will slow down the training process. VOMPS overcomes the above two weaknesses and achieves the $O(1/\epsilon^{2/3})$ sample complexity by applying increasing weights strategy and an automatically adjusted stepsize strategy that proportions to iterations. These techniques are not used in the original STORM method. As a result, VOMPS is the first variance reduced memory-efficient off-policy method that achieves the optimal sample complexity, which matches the lower-bound provided in~\cite{arjevani2019lower}.

\begin{table}[ht]
 \begin{center}
 \begin{tabular}{c c c c c}
 \toprule
 {Algorithms}   & Objective & {Sample Complexity}  & {Off-Policy?} 
  & {Required Batch} \\ 
 \toprule
 SVRPG \citep{svrpg}  & $J_\pi$ &$O(1/{\epsilon ^2})$ &$\times$   & $O(1/{\epsilon})$ \\
 SVRPG \citep{xu2019improved} & $J_\pi$ &$O(1/{\epsilon ^{5/3}})$ &$\times$ 
  & $O(1/{\epsilon^{2/3}})$ \\
 SRVR-PG \citep{sarahpg} & $J_\pi$ & $O(1/{\epsilon ^{3/2}})$ &$\times$ 
  & $O(1/{\epsilon^{1/2}})$ \\
 ACE-STORM (This paper)  & $J_{\mu}$ & $O(1/{\epsilon ^{3/2}})$ & \checkmark   & $\times$ \\
 VOMPS (This paper) & $J_{\hat \gamma}$ & $O(1/{\epsilon ^{3/2}})$ & \checkmark   & $\times$\\
 \bottomrule
 \end{tabular}
 \end{center}
 \begin{footnotesize}
 \caption{\label{table:pg_rate}{\footnotesize Comparison on convergence rate of different algorithms when $\|\nabla J(\theta)\|_2^2 \leq \epsilon$. The $\times$ in ``Required Batch'' means that no mini-batch is needed, aka, the algorithm is memory efficient.}}
 \end{footnotesize}
 \end{table}
\begin{remark}
Theorem~\ref{thm:vomps} indicates that VOMPS enjoys the same convergence rate as the state-of-the-art algorithms together with SRVR-PG~\citep{sarahpg}. A summary of state-of-the-art convergence rate is summarized in Table~\ref{table:pg_rate}.
It should be noted that unlike other algorithms that optimize $J_\pi$, the objective function of VOMPS is $J_{\hat \gamma}$. However, these two objective functions only differ in constants in the on-policy setting. Thus, it is a fair comparison of the convergence rate. 
\end{remark}

\subsection{ACE-STORM Algorithm Design}
As an extension, we also propose the \textbf{A}ctor-\textbf{C}ritic with \textbf{E}mphatic weighting and \textbf{STO}chastic \textbf{R}ecursive \textbf{M}omentum (ACE-STORM) algorithm, which is an integration of the ACE algorithm~\citep{imani2018off} and STORM~\citep{storm}. Similar to discussed above, the objective of ACE-STORM is also a special case of that in Algorithm~\ref{alg:vomps} by setting $\hat \gamma = 0$. The pseudo-code of the ACE-STORM Algorithm is presented in the Appendix due to space constraints.


\section{Experiments and Results}
\label{sec:exp}
The experiments are conducted to investigate the following questions empirically. i) How do VOMPS\&ACE-STORM compare with state-of-art off-policy policy gradient methods, such as GeoffPAC~\citep{geoffpac}, ACE, DDPG~\citep{lillicrap2015continuous}, and TD3~\citep{fujimoto2018addressing}?  ii) How do VOMPS\&ACE-STORM compare with on-policy variance-reduced policy gradient methods, e.g., SVRPG~\citep{svrpg} and SRVR-PG~\citep{sarahpg}? iii) How is VOMPS\&ACE-STORM resilient to action noise in policy gradient methods?

Since the tasks for original domains are episodic, e.g., \texttt{CartPoleContinuous-v0}, \texttt{Hopper-v2} and \texttt{HalfCheetah-v2}, we modify them as continuing tasks following ~\citep{geoffpac}--- the discount factor $\gamma$ is set to $0.99$ for all non-terminal states and $0$ for terminal states. The environment is reset to the initial state if the agent reaches the terminal state.
Therefore, simulation-based cumulative rewards (Monte Carlo return) by executing the (learned) policy $\pi$ is used as the performance metric, while results of episodic return are also provided in the Appendix.
All the curves in the results are averaged over $10$ runs, where the solid curve indicates the mean and the shaded regions around the mean curve indicate standard deviation errors. To better visualize the plots, curves are smoothed by a window of size
$20$. Shorthands ``1K'' represents $10^3$, and ``1M'' represents $10^6$. 
In the off-policy setting, the behavior policy $\mu$ follows a fixed uniform distribution. VOMPS, ACE-STORM, GeoffPAC, and ACE have the same critic component in all experiments for a fair comparison.

\subsection{Tabular off-policy policy gradient}
We first compare the performance of ACE, GeoffPAC, ACE-STORM, and VOMPS on the two-circle MDP domain~\citep{imani2018off,geoffpac} in terms of their dynamic and asymptotic solutions. In the two-circle MDP, there are a finite number of states, and an agent only decides at state \texttt{A} on either transitioning to state \texttt{B} or state \texttt{C}, whereas the transitions at other states will always be deterministic. The discount factor $\gamma = 0.6$ and rewards are $0$ unless specified on the edge as shown in Fig.~\ref{fig:circle}.
\begin{figure*}[htb!]
\begin{subfigure}{.5\textwidth}
  \centering
    \includegraphics[height=2.5cm,width=.6\textwidth]{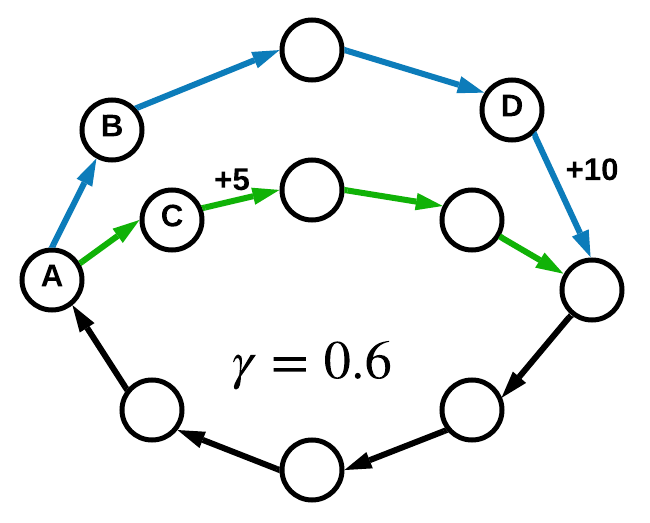}
  \caption{Two-circle MDP}
  \label{fig:circle}
\end{subfigure}
\begin{subfigure}{.5\textwidth}
  \centering
    \includegraphics[height=3.0cm,width=.8\textwidth]{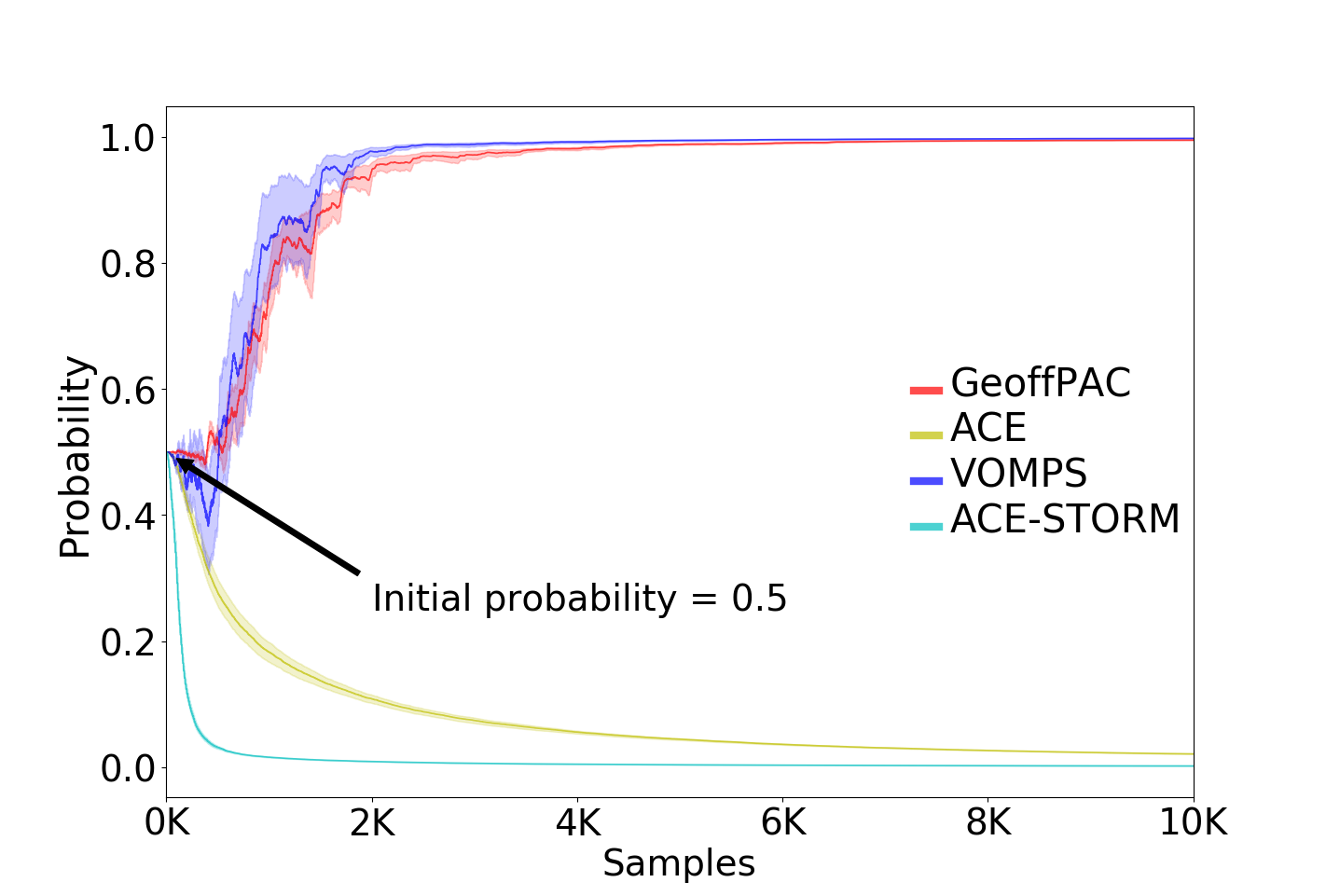}
  \caption{The probability of transitioning from A to B}
  \label{fig:circle_result}
\end{subfigure}
\caption{The two-circle MDP}
\end{figure*}
The algorithms for this domain, GeoffPAC, ACE, VOMPS, and ACE-STORM, are implemented with a tabular version, where value function and density ratio function is computed via dynamic programming. The behavior policy $\mu$ follows a uniform distribution, and $\pi(A \rightarrow B)$, the probability from A to B under the target policy $\pi$ is reported in  Fig.~\ref{fig:circle_result}. 
As shown in Fig.~\ref{fig:circle_result}, VOMPS and GeoffPAC gradually choose to transition from A to B so that the agent would take the route with blue color and obtain a reward of $+10$. Compared with GeoffPAC, VOMPS converges faster.
Both ACE-STORM and ACE move from A to C, and ACE-STORM converges faster than ACE. Moving from A to C is an inferior solution since the agent will take the route with green color and fail to obtain a higher reward.
The difference between asymptotic solutions of GeoffPAC/VOMPS and ACE/ACE-STORM is due to the difference between the objective functions $J_{\hat \gamma}, J_\mu$, and the difference in the training process is due to the STORM component integrated into VOMPS and ACE-STORM.

\begin{figure*}[htb!]
\centering
\begin{subfigure}{.44\textwidth}
  \centering
  \includegraphics[height=3.5cm,width=1.\textwidth]{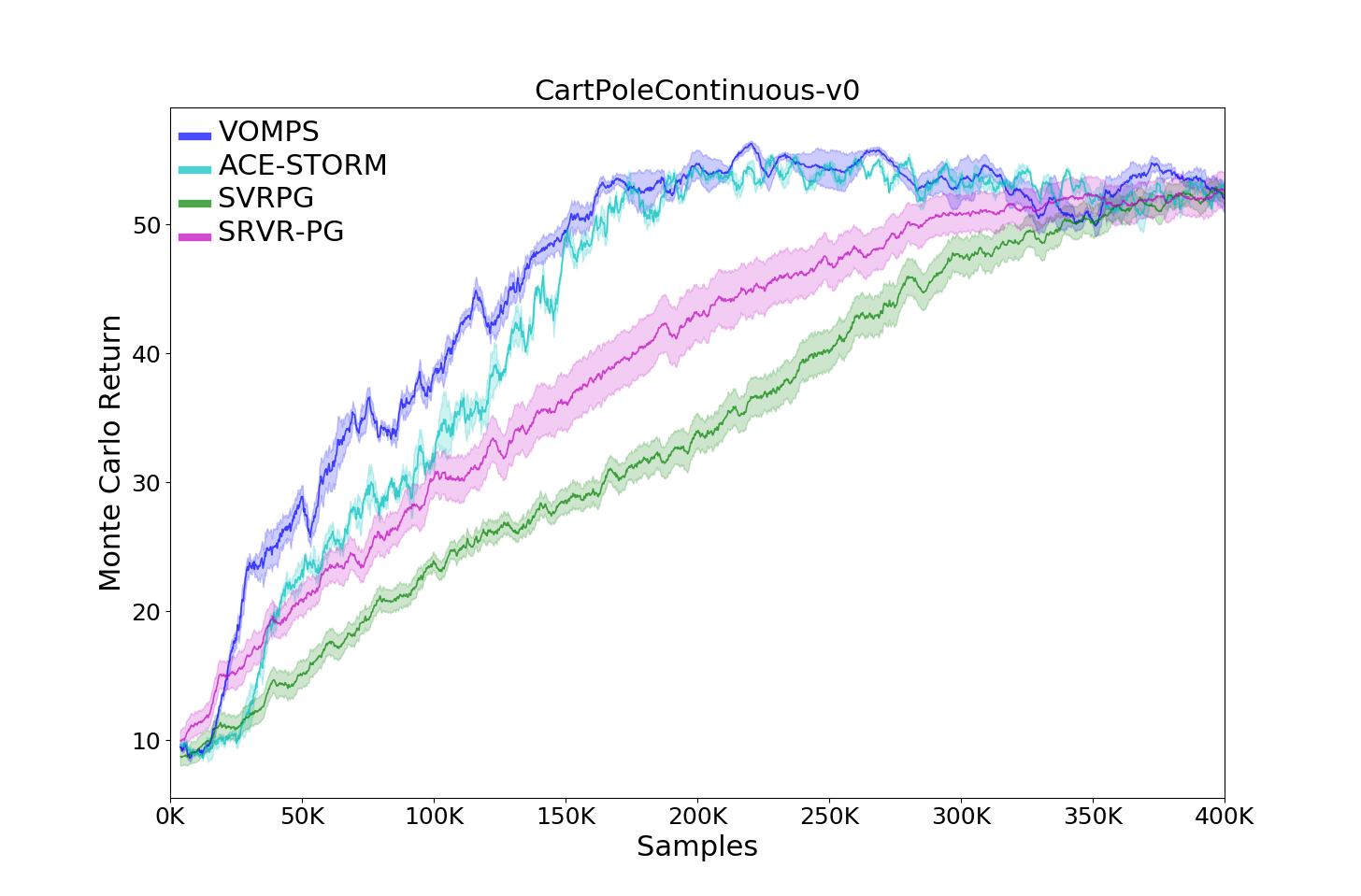}
  \caption{Comparison with on-policy methods}
  \label{fig:cartpole:onpol}
\end{subfigure}
\begin{subfigure}{.44\textwidth}
  \centering
  \includegraphics[height=3.5cm,width=1.\textwidth]{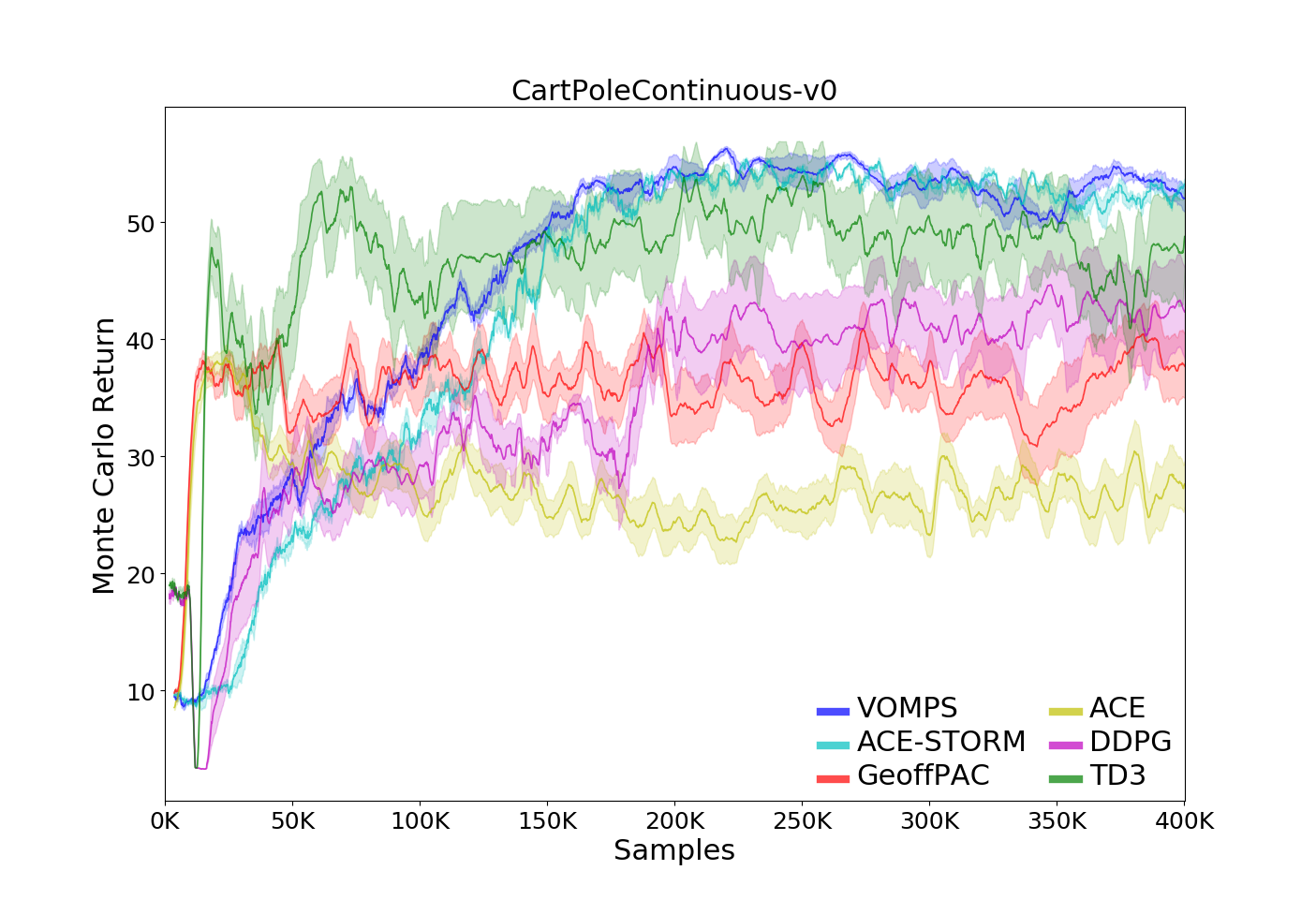}
  \caption{Comparison with off-policy methods}
  \label{fig:cartpole:offpol}
\end{subfigure}
\caption{Results on \texttt{CartPoleContinuous-v0}}
\end{figure*}

\subsection{Classic Control}
We use \texttt{CartPoleContinuous-v0} for CartPole domain, which has a continuous action space within the range of $[-1,1]$.
A near-optimal policy can reach a Monte-Carlo return at the level of $57$ within a fixed horizon of $200$ timesteps.
As shown in Fig.~\ref{fig:cartpole:onpol}, VOMPS and ACE-STORM learn the near-optimal policy with around $200$K samples, while SVRPG and SRVR-PG need more than $400$K samples with larger dynamic variances. As Fig.~\ref{fig:cartpole:offpol} shows, ACE, GeoffPAC, and DDPG do not perform well in this domain. Although TD3 seems to learn faster at the beginning, it reaches an inferior solution with a mean return around $50$ a higher variance than VOMPS and ACE-STORM.

\subsection{Mujoco Robot Simulation}
Experiments are also conducted on two benchmark domains provided by OpenAI Gym, including \texttt{Hopper-v2} and \texttt{HalfCheetah-v2}. 
As shown in Fig.~\ref{fig:hopper:onpol},~\ref{fig:hopper}, both GeoffPAC and VOMPS can achieve higher Monte Carlo returns than other methods and converge faster within $1$M samples on \texttt{Hopper-v2}. Compared with GeoffPAC, the learning curve of VOMPS is smoother and has a smaller variance. 
The results on \texttt{HalfCheetah-v2} are shown in Fig.~\ref{fig:halfcheetah:onpol},~\ref{fig:halfcheetah}. Fig.~\ref{fig:halfcheetah:onpol} indicates that VOMPS and ACE-STORM outperform SVRPG and SRVR-PG by a large margin, and Fig.~\ref{fig:halfcheetah} demonstrates that VOMPS and ACE-STORM achieve a similar performance of GeoffPAC/DDPG/TD3, with obviously smaller variances. We also observe that ACE does not perform well in general, and DDPG has a very large variance in these two domains.

In addition, a $20\%$ action noise is added to both the learning process and evaluation process in order to compare the noise resistance ability of different approaches (aka, the action is multiplied by a factor of $1\pm 0.2\chi$, where $\chi$ is drawn from a $[0,1]$-range uniform distribution). As shown in Fig.~\ref{fig:hopper:noise:onpol},~\ref{fig:hopper:noise:offpol},~\ref{fig:halfcheetah:noise:onpol},~\ref{fig:halfcheetah:noise:offpol}, compared with results under the original noise-free setting, VOMPS, ACE-STORM, SVRPG, and SRVR-PG tend to be insensitive to disturbances than other methods, which validates the effectiveness of the stochastic variance reduction component of these algorithms. In particular, VOMPS and ACE-STORM appear to be empirically the most noise-resistant in these two domains.

\begin{footnotesize}
\begin{figure*}[htb!]
\begin{subfigure}{.24\textwidth}
  \centering
  \includegraphics[width = 1\textwidth]{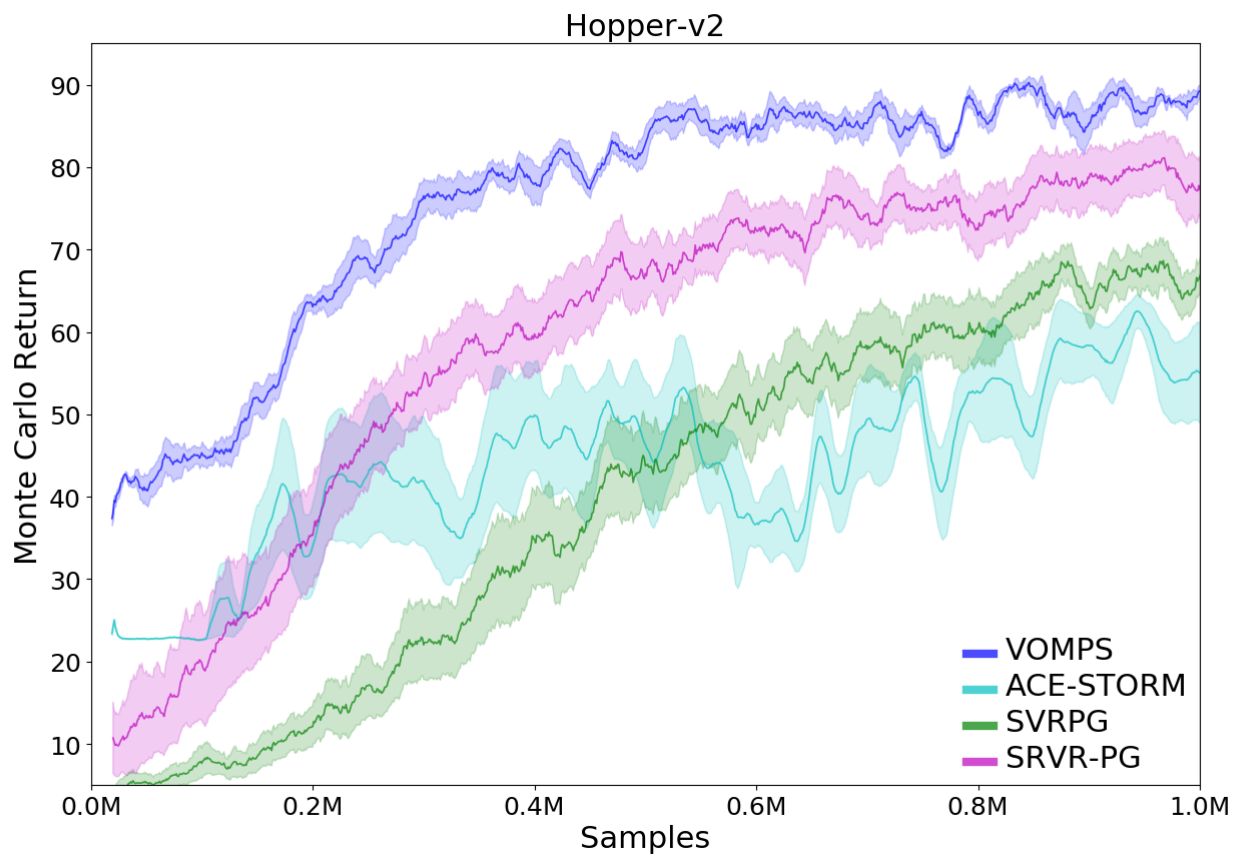}
  \caption{Hopper}
  \label{fig:hopper:onpol}
\end{subfigure}
\begin{subfigure}{.24\textwidth}
  \centering
  \includegraphics[width = 1\textwidth]{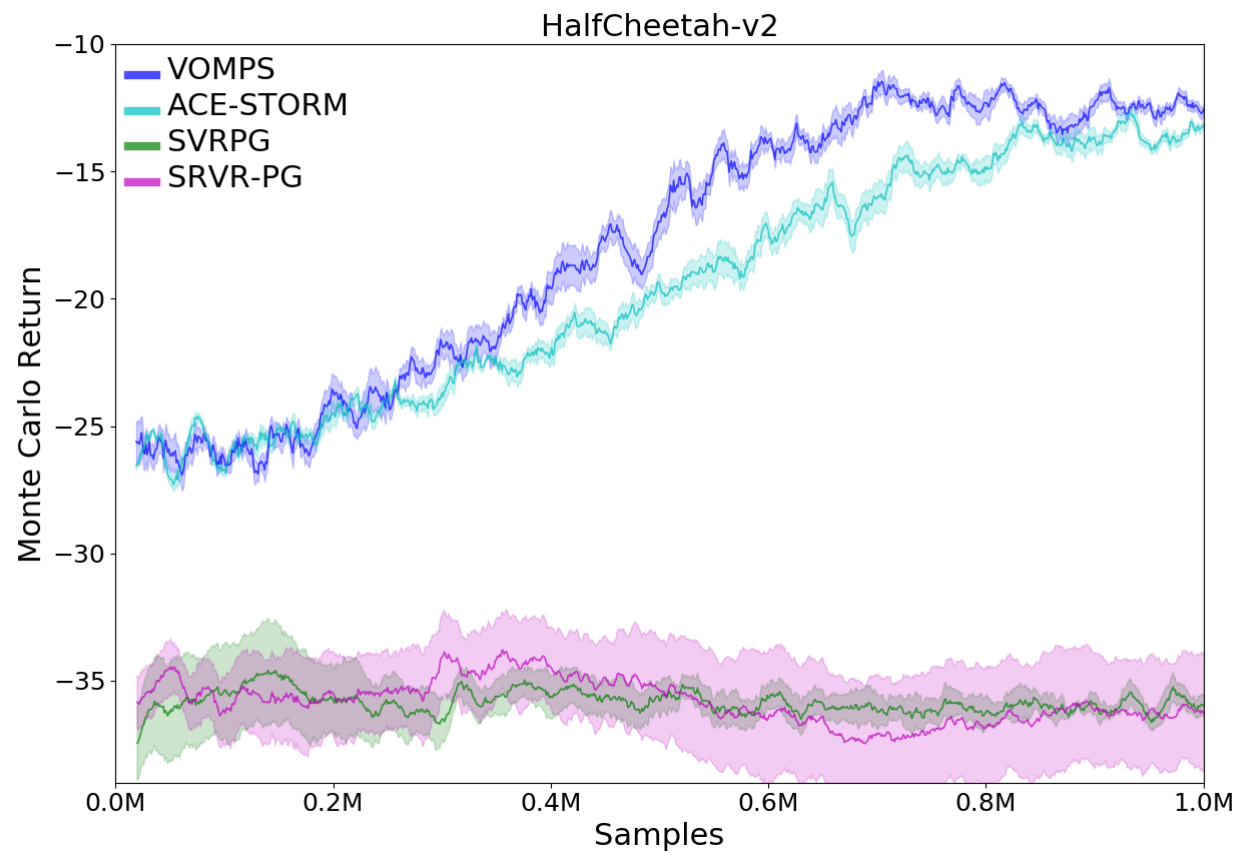}
    \caption{HalfCheetah}
    \label{fig:halfcheetah:onpol}
\end{subfigure}
\begin{subfigure}{.24\textwidth}
  \centering
  \includegraphics[width = 1\textwidth]{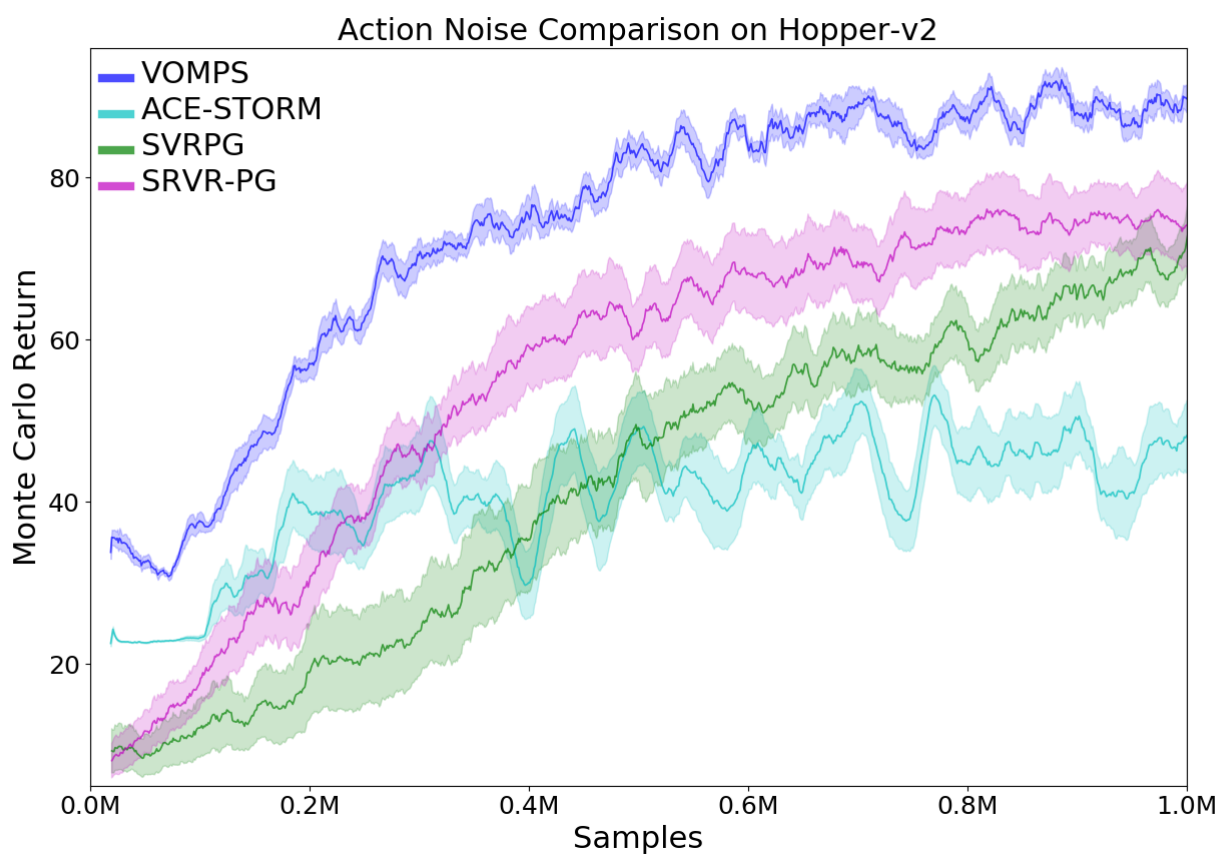}
  \caption{Hopper (action noise)}
  \label{fig:hopper:noise:onpol}
\end{subfigure}
\begin{subfigure}{.24\textwidth}
  \centering
  \includegraphics[width = 1\textwidth]{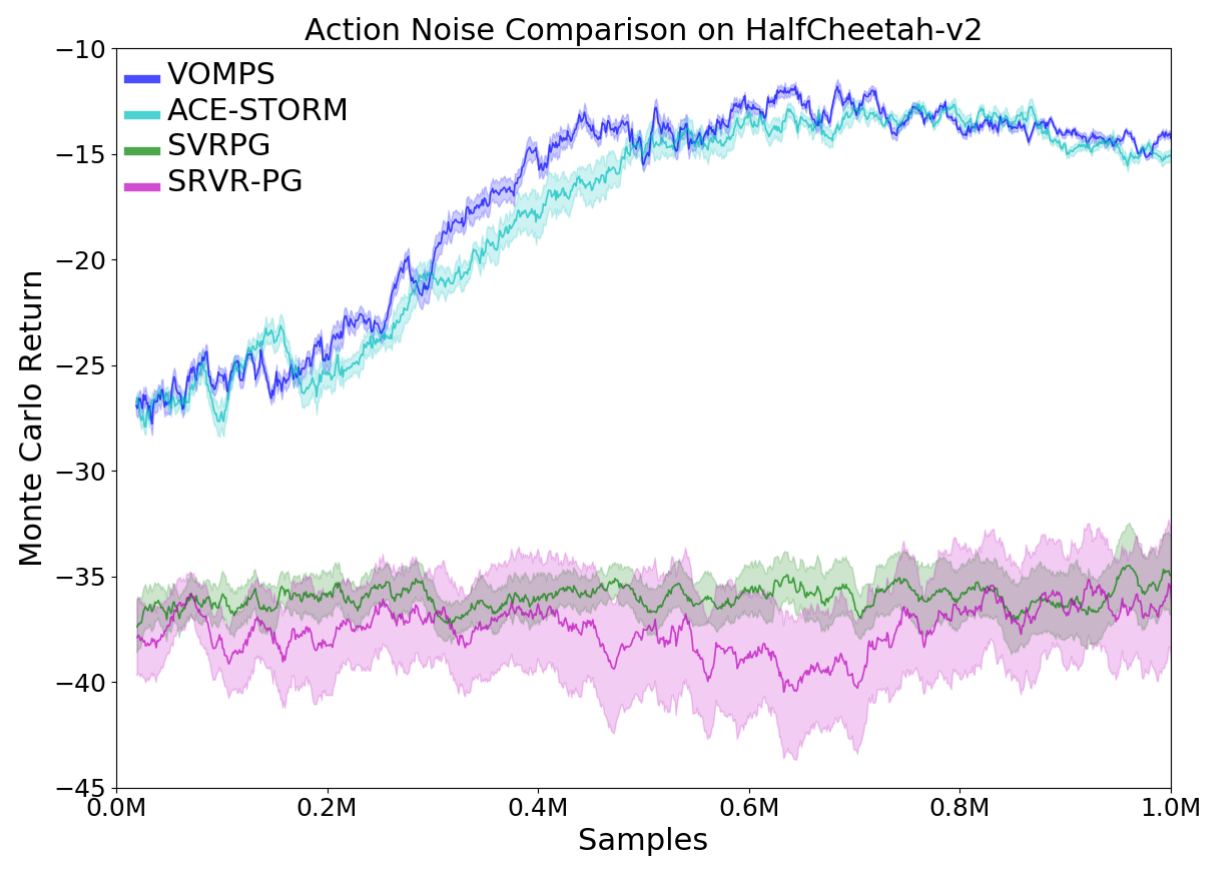}
  \caption{{\footnotesize HC (action noise)}}
  \label{fig:halfcheetah:noise:onpol}
\end{subfigure}
\caption{Comparison with on-policy PG methods (Mujoco), ``HC'' is short for HalfCheetah.}
\end{figure*}

\begin{figure*}[htb!]
\begin{subfigure}{.24\textwidth}
  \centering
  \includegraphics[width = 1\textwidth]{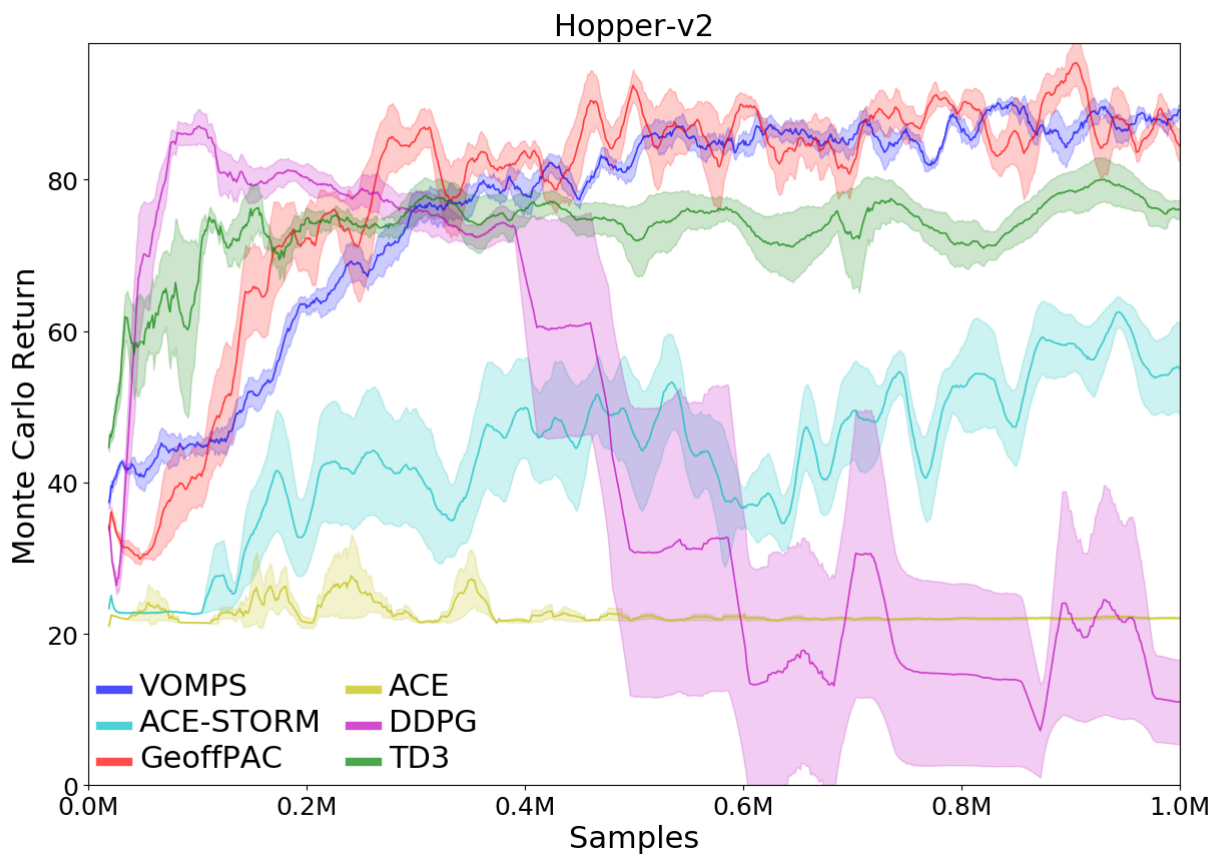}
  \caption{Hopper}
  \label{fig:hopper}
\end{subfigure}
\begin{subfigure}{.24\textwidth}
  \centering
  \includegraphics[width = 1\textwidth]{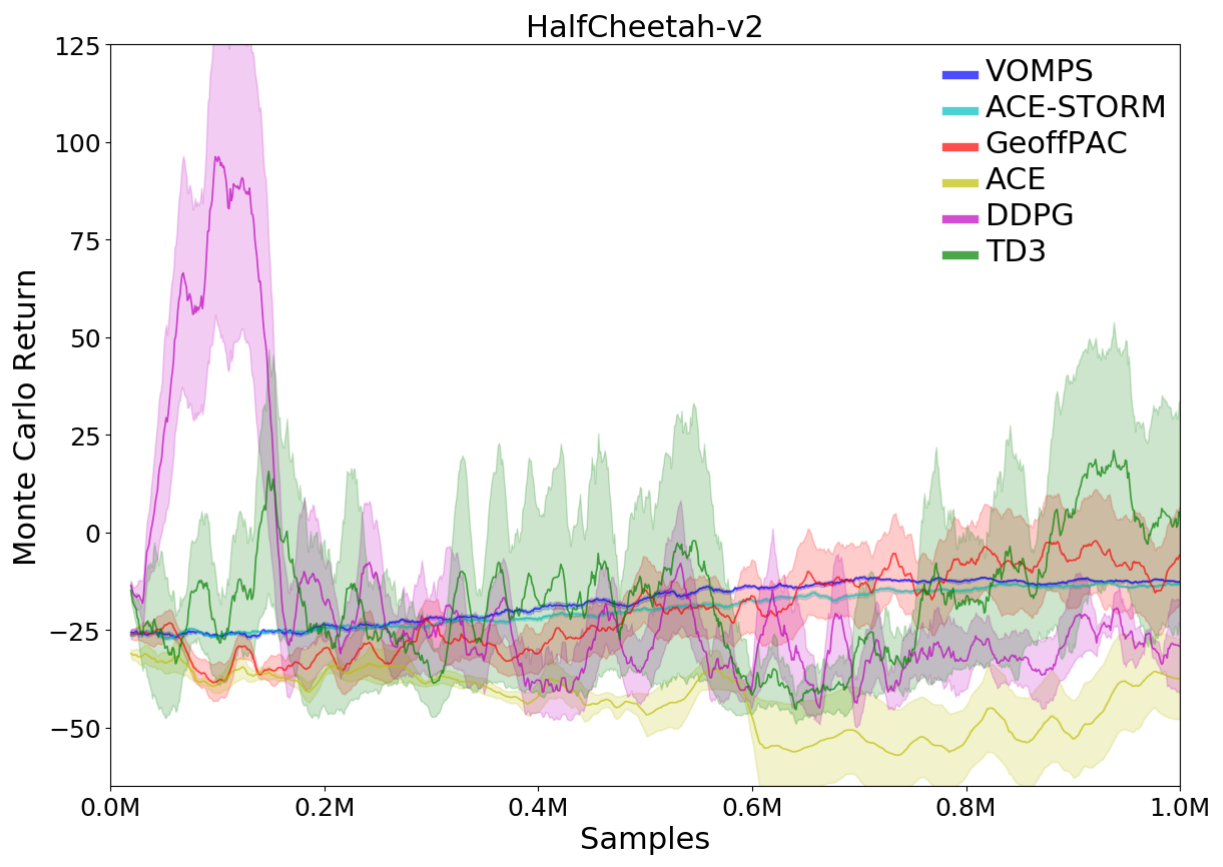}
    \caption{HalfCheetah}
    \label{fig:halfcheetah}
\end{subfigure}
\begin{subfigure}{.24\textwidth}
  \centering
  \includegraphics[width = 1\textwidth]{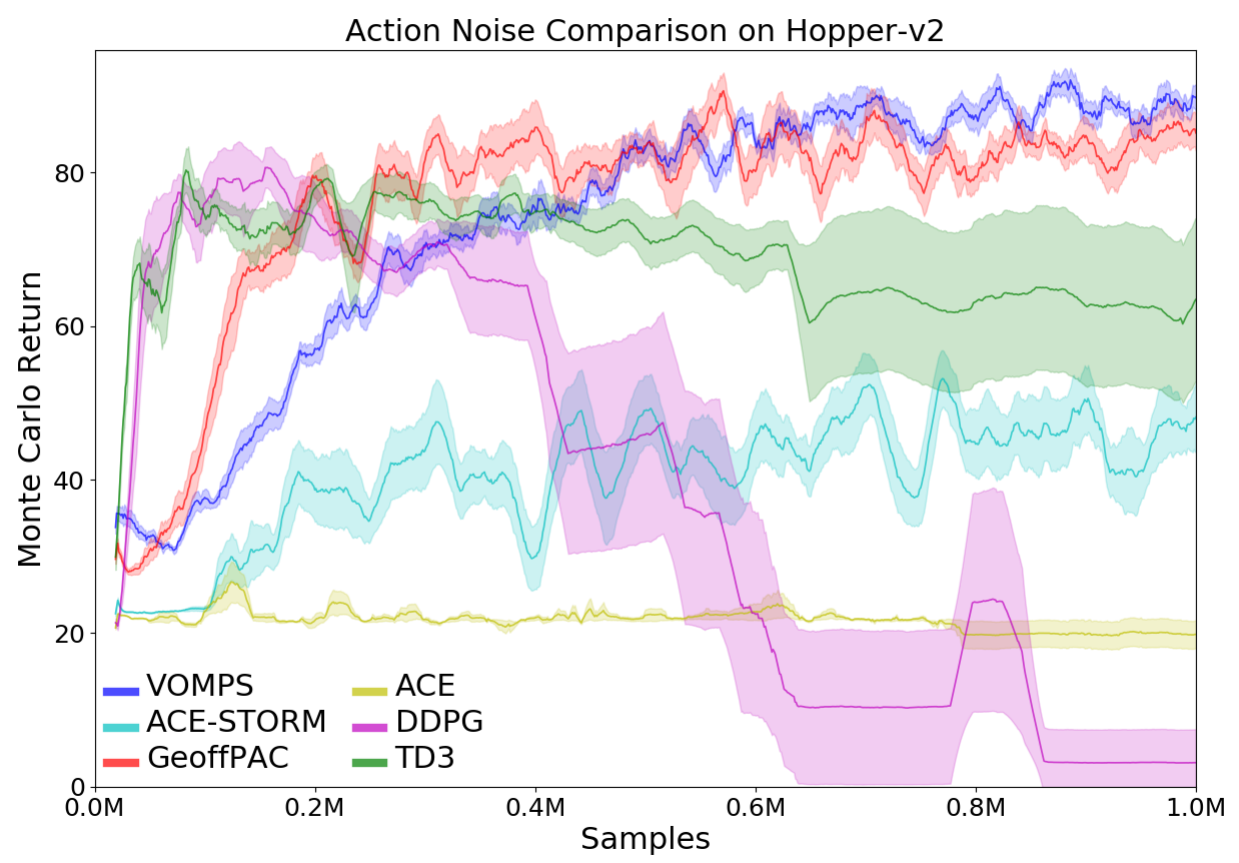}
  \caption{Hopper (action noise)}
  \label{fig:hopper:noise:offpol}
\end{subfigure}
\begin{subfigure}{.24\textwidth}
  \centering
  \includegraphics[width = 1\textwidth]{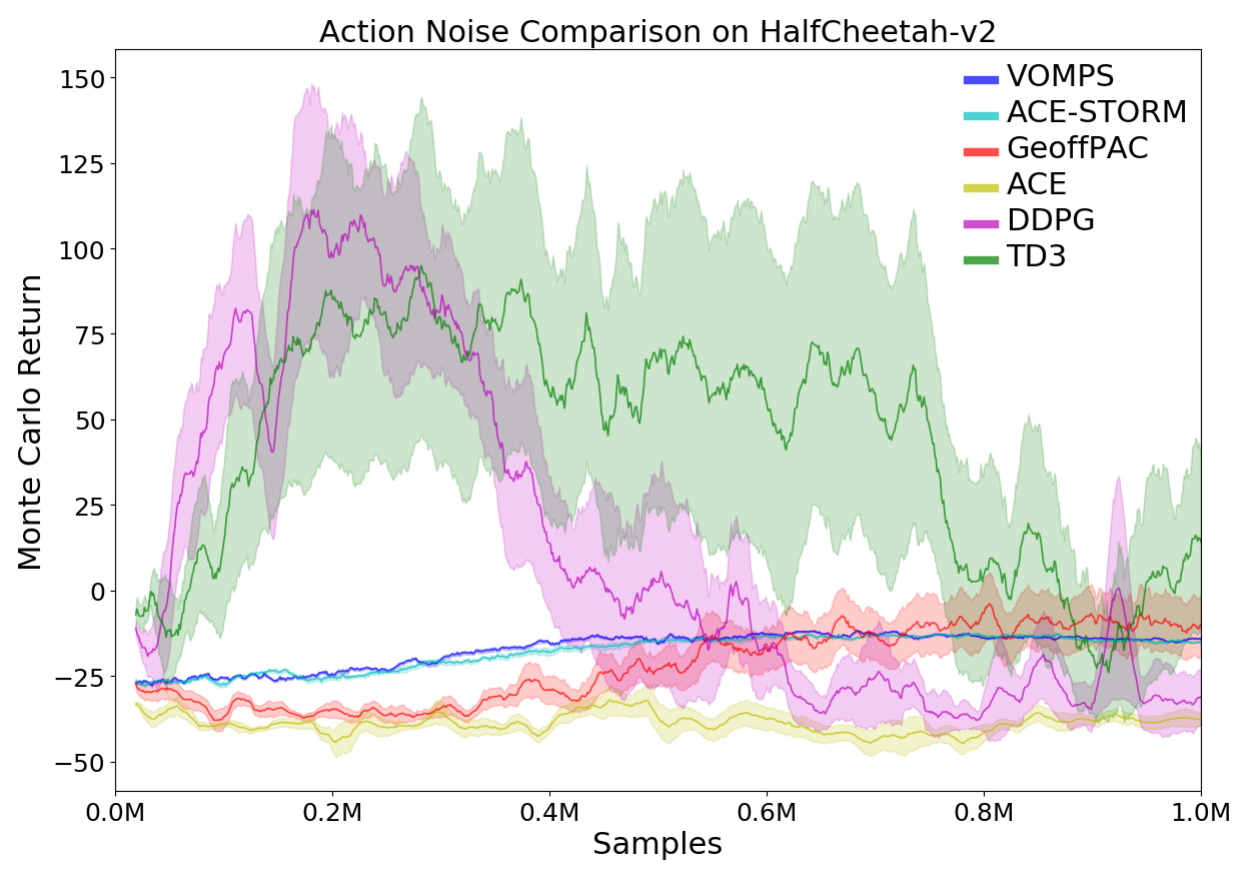}
    \caption{{\footnotesize HC (action noise)}}
  \label{fig:halfcheetah:noise:offpol}
\end{subfigure}
\caption{Comparison with off-policy PG methods (Mujoco), ``HC'' is short for HalfCheetah.}
\end{figure*}
\end{footnotesize}

\section{Related Work}
\label{sec:related}
Policy gradient methods and the corresponding actor-critic algorithm \citep{sutton2000policy,konda2002thesis} are popular policy search methods in RL, especially for continuous action setting.
However, this class of policy search algorithms suffers from large variance~\citep{pg:robotics:peters2006,deisenroth2013survey}. Several approaches have been proposed to reduce variance in policy search.
The first method family is to use control variate method, such as baseline removal~\citep{sutton2018reinforcement}, to remove a baseline function in the policy gradient estimation~\citep{weaver2001optimal,greensmith2002variance,gu2017q,tucker2018mirage}. 
The second method family is based on tweaking batch size, stepsize, and importance ratio used in policy search. In this research line, \citep{pirotta2013adaptive} proposed using an adaptive step size to offset the effect of the policy variance. \citet{pirotta2013adaptive, papini2017adaptive} studied the adaptive batch size and proposed to optimize the adaptive step size and batch size jointly, and \cite{metelli2018policy} investigated reducing variance via importance sampling.
 The third branch of methods is based on the recently developed \textit{stochastic variance reduction}~\citep{svrg,allen2016variance,reddi2016nips} methods as discussed above. Several variance-reduced policy gradient methods were proposed in this direction, such as SVRPG~\citep{svrpg}, SRVR-PG~\citep{sarahpg}, etc.

\section{Conclusion}\label{sec:conclusion}
In this paper, we present off-policy convergent, memory-efficient, and variance-reduced policy search algorithms by leveraging emphatic-weighted policy search and stochastic recursive momentum-based variance reduction. Experimental study validates the performance of the proposed approaches compared with existing on-policy variance-reduced policy search methods and off-policy policy search methods under different settings.
Future work along this direction includes integrating with baseline removal methods for further variance reduction and investigating algorithmic extensions to risk-sensitive policy search and control.

\bibliographystyle{apalike}
\bibliography{reference}

\clearpage
\begin{center}
{\Large \textbf{Appendix}}
\end{center}

\appendix
\section{Hyperparameters in Algorithm~\ref{alg:vomps}}
Hyper-parameters are presented below in the order of four main components--- updating the critic, the density ratio, the emphatic weights, and the actor. $\alpha_{\nu} \in [0,1]$ is the stepsize in the critic update; $\alpha_{\psi} \in [0,1]$ is the stepsize in the density ratio update; $\lambda^{(1)} \in [0,1]$, $\lambda^{(2)} \in [0,1]$ and $\hat{\gamma} \in [0,1)$ can be found more details in Appendix~\ref{sec:app:emphatic} for the emphatic weights update; $k$, $w$, and $\beta$ are inherited from STORM for the actor update. By default, $w$ is set as $10$ and $\beta =100$.

\section{Emphatic weights update component of GeoffPAC~\citep{geoffpac}}
\label{sec:app:emphatic}
 Figure~\ref{fig:geoffpac} contains the updates for the emphatic weights in GeoffPAC. In this figure, $\lambda^{(1)}$ and $\lambda^{(2)}$ are parameters that are used for bias-variance tradeoff, $C(s) = \frac{d_{\hat{\gamma}}(s)}{d_\mu(s)}$ is the density ration function (\citealt{gelada2019off} call it covariate shift), and $i(s)$ is the intrinsic interest function that is defined from the extrinsic interest function $\hat{i}(s)$ as $i(s)=C(s) \hat{i}(s)$. In practice, $\hat{i}(s) =1$. At time-step $t$, $F^{(1)}_t$ and $F^{(2)}_t$ are the follow-on traces, $M^{(1)}_t$ and $M^{(2)}_t$ are the emphatic weights, $I_t$ is the gradient of the intrinsic interest, $\delta_t$ is the temporal-difference (TD) error, and finally $Z_t$ is an unbiased sample of $\nabla J_{\hat{\gamma}}$. 
 For more details about these parameters and their update formulas, we refer the reader to the GeoffPAC paper~\citep{geoffpac}.
\begin{figure}[tbh]
\center
\fbox{
\begin{tabular}{p{13.2cm}}
  \textbf{HYPER-PARAMETER}: $\lambda^{(1)}, \lambda^{(2)}$.\\
  \textbf{INPUT}:$F^{(1)}_{t-1}, F^{(2)}_{t-1}, \rho_{t-1}, \rho_{t}, C(s_{t};\psi_{t}), V(s_{t};\nu_t), \delta_{t}, \hat{i}(s_{t})$.\\
  \textbf{OUTPUT}:$F^{(1)}_t, M^{(1)}_t, I_t, F^{(2)}_t, M^{(2)}_t,  Z_t(a_t,s_t;\theta_{t})$.
          \\Compute $F^{(1)}_t = \gamma \rho_{t-1}F^{(1)}_{t-1} + \hat{i}(s_{t}) C(s_{t};\psi_{t})$.  
          \\Compute $M^{(1)}_t = (1 - \lambda^{(1)}) \hat{i}(s_{t}) C(s_{t};\psi_{t}) + \lambda^{(1)} F^{(1)}_t$.  
          \\Compute $I_t = C(s_{t-1};\psi_{t-1}) \rho_{t-1} \nabla_{\theta} \log \pi(a_{t-1} | s_{t-1};\theta_{t-1})$.
          \\Compute $F^{(2)}_t = \hat{\gamma} \rho_{t-1} F^{(2)}_{t-1} + I_t$.
          \\Compute $M^{(2)}_t = (1 - \lambda^{(2)})I_t + \lambda^{(2)} F^{(2)}_t $. 
          \\Compute $Z_t(a_t,s_t;\theta_{t}) = \hat{\gamma} \hat{i}(s_{t}) V(s_{t};\nu_t)M^{(2)}_t + \rho_{t} M^{(1)}_t \delta_{t} \nabla_{\theta} \log \pi(a_{t} | s_{t};\theta_{t})$.
\end{tabular}}
\caption{Emphatic weights update component of GeoffPAC~\citep{geoffpac}}
\label{fig:geoffpac}
\end{figure}

\section{ACE-STORM Algorithm}
The pseudo-code of ACE-STORM is shown in Algorithm~\ref{alg:ace-storm}.

\begin{algorithm}[hb!]
\caption{ACE-STORM}
\label{alg:ace-storm}
\begin{algorithmic}
\STATE $V$: value function parameterized by $\nu$\;
\STATE $\pi$: policy function parameterized by $\theta$\; 
\STATE \textbf{Input}: Initial parameters $\nu_0$ and $\theta_0$. Initialize $F^{(1)}_{-1} = 0$, $\rho_{-1} = 1$, $i(\cdot) = 1$,  and hyper-parameters $\lambda^{(1)}$, $k$, $w$, $\beta$ and $\alpha_{\nu}$.
\FOR{timestep $t=0$ to $T$}  
  \STATE Sample a transition $S_t$, $A_t$, $R_t$, $S_{t+1}$ according to behavior policy $\mu$.
  \STATE Compute $\delta_{t} = R_{t} + \gamma V(S_{t+1}; \nu_{t}) - V(S_{t}; \nu_{t})$
  \STATE Update the parameter for value function:
    $\nu_{t+1} = \nu_{t} + \alpha_{\nu} \delta_{t} \nabla_{\nu}V(S_{t}; \nu_{t})$
        
  \STATE Compute $F^{(1)}_t = \gamma \rho_{t-1} F^{(1)}_{t-1} + i(S_t)$
  \STATE Compute $M^{(1)}_t = (1-\lambda^{(1)}) i(S_t) + \lambda^{(1)} F^{(1)}_t$

  \STATE Compute $Z^{(1)}_t(A_t,S_t;\theta_{t}) = \rho_{t} M^{(1)}_t \delta_{t} \nabla_{\theta} \log \pi(A_{t} |S_{t};\theta_{t})$.
        
  \STATE Compute $G_{t} = ||Z^{(1)}_t(A_t,S_t;\theta_{t})||$.
  \STATE Compute $\alpha_{t} = \beta \eta^2_{t-1}$
    \STATE Compute $Z^{(1)}_t(A_t,S_t;\theta_{t-1}) = \rho_{t} M^{(1)}_t \delta_{t} \nabla_{\theta} \log \pi(A_{t}|S_{t};\theta_{t-1})$.
    \STATE Compute $g_{t} = Z^{(1)}_t(A_t,S_t;\theta_{t}) + (1 - \alpha_t)\big(g_{t-1} - Z^{(1)}_t(A_t,S_t;\theta_{t-1}) \big)$.
  \STATE Compute $\eta_t = \frac{k}{(w + \sum^{t}_{i=1} {G^2_t})^\frac{1}{3}}$.
  \STATE Update  the parameter for the actor: $\theta_{t+1} = \theta_t + \eta_t g_{t}$
\ENDFOR
\STATE \textbf{Output I}: Parameters $\nu_{T+1}$, $\theta_{T+1}$.
\STATE \textbf{Output II}:Parameters $\nu_{T+1}$, $\theta_\tau$, where $\tau$ is sampled with a probability of $p(\tau  = t)\propto \frac{1}{\eta_t^2}$.
\end{algorithmic}
\end{algorithm}

\section{Comparison of Stochastic Variance Reduction Methods}
\label{sec:compare-svr}
This table is adapted from \citep{storm}.
\begin{table}[ht]
\begin{footnotesize}
\begin{center}
\begin{tabular}{ccccc}
\toprule
Algorithms & &Sample Complexity  & Reference Sets Needed?  \\
\toprule
\multirow{2}{*}{SVRG} &\citep{reddi2016icml}  &\multirow{2}{*}{$O(n^{2/3}/\epsilon)$}& \multirow{2}{*}{$O(1/\epsilon)$} \\
&\citep{allen2016variance}  & &\\ \midrule
SARAH &\citep{sarah,nguyen2017stochastic} &  $O(n+1/\epsilon^2)$  & \checkmark\\ \midrule
SPIDER &\citep{spider}  &$O(1/\epsilon^{3/2})$  & \checkmark\\ \midrule
STORM &\citep{storm} &$O(1/\epsilon^{3/2})$ &  $\times$\\
\bottomrule
\end{tabular}
\end{center}
\caption{Comparison of convergence rates to
achieve $||
\nabla J(x)||^2\leq \epsilon$ for \textit{nonconvex} objective functions.}
\label{table:complexity_2}
\end{footnotesize}
\end{table}

\section{Proof of Theorem~\ref{thm:vomps}}
Before conducting the proof, we first denote  $\epsilon_t$: $\epsilon_t = g_t - \nabla J_{\hat{\gamma}}(\theta_t)$.
\begin{lemma}
\label{lemma:obj}
Suppose $\eta_t\le \frac{1}{4L}$ for all $t$. Then
\begin{align}
    \mathbb{E}\big[J_{\hat{\gamma}}(\theta_{t}) - J_{\hat{\gamma}}(\theta_{t+1}) \big]
\le \mathbb{E} \big[- \eta_t/4 \|\nabla J_{\hat{\gamma}}(\theta_t)\|^2 + 3\eta_t/4 \|\epsilon_t\|^2 \big]
\end{align}
\end{lemma}

\begin{proof}[Proof of Lemma~\ref{lemma:obj}] According to the smoothness of $J_{\hat{\gamma}}$,
\begin{equation}
\begin{aligned}
\big[-J_{\hat{\gamma}}(\theta_{t+1})]
&\leq \E[- J_{\hat{\gamma}}(\theta_t)  -\nabla J_{\hat{\gamma}}(\theta_t)\cdot \eta_tg_t + \frac{L\eta_t^2}{2}\|g_t\|^2 \big]\\
&=\E[- J_{\hat{\gamma}}(\theta_t) - \eta_t\|\nabla J_{\hat{\gamma}}(\theta_t)\|^2 - \eta_t\nabla J_{\hat{\gamma}}(\theta_t)\cdot\epsilon_t + \frac{L\eta_t^2}{2}\|g_t\|^2 \big] \\
&\leq \E[- J_{\hat{\gamma}}(\theta_t) - \frac{\eta_t}{2}\|\nabla J_{\hat{\gamma}}(\theta_t)\|^2 +\frac{\eta_t}{2}\| \epsilon_t\|^2 + \frac{L\eta_t^2}{2}\|g_t\|^2 \big] \\
&\leq  \E[- J_{\hat{\gamma}}(\theta_t) - \frac{\eta_t}{2}\|\nabla J_{\hat{\gamma}}(\theta_t)\|^2 +\frac{\eta_t}{2}\| \epsilon_t\|^2  + L\eta_t^2\| \epsilon_t\|^2 + L\eta_t^2\|\nabla J_{\hat{\gamma}}(\theta_t)\|^2 \big]\\
&\leq  \E[- J_{\hat{\gamma}}(\theta_t) - \frac{\eta_t}{2}\|\nabla J_{\hat{\gamma}}(\theta_t)\|^2 +\frac{3\eta_t}{4}\| \epsilon_t\|^2 + \frac{\eta_t}{4}\|J_{\hat{\gamma}}(\theta_t)\|^2 \\
\end{aligned}
\end{equation}

\end{proof}

The following technical observation is key to our analysis: it provides a recurrence that enables us to bound the variance of the estimates $g_t$.
\begin{lemma}
\label{lemma:epsilonrecursion}
With the notation in Algorithm, we have
\begin{align}
& \mathbb{E} \big[\|\epsilon_t\|^2/\eta_{t-1} \big] \\
\leq &\mathbb{E} \big[2 \beta^2 \eta_{t-1}^3 \sigma^2 + (1-\alpha_t)^2 (1+4 L^2 \eta_{t-1}^2)\|\epsilon_{t-1}\|^2/\eta_{t-1}+4 (1-\alpha_t)^2 L^2 \eta_{t-1}\|\nabla J_{\hat{\gamma}}(\theta_{t-1})\|^2 \big] .
\end{align}
\end{lemma}
The proof of Lemma~\ref{lemma:epsilonrecursion} is identical to the proof of Lemma 2 in \citep{storm}.

\begin{proof}[Proof of Theorem~\ref{thm:vomps}]
We first construct a Lyapunov function of  $\Phi_t = J_{\hat{\gamma}}(\theta_t) + \frac{1}{32L^2 \eta_{t-1}}\|\epsilon_t\|^2$. We will upper bound $\Phi_{t+1} - \Phi_t$ for each $t$, which will allow us to bound $\Phi_T$ in terms of $\Phi_1$ by summing over $t$. First, observe that since $w \geq (4Lk)^3$, we have $\eta_{t}\leq \frac{1}{4L}$. Further, since $\alpha_{t+1}=\beta \eta_t^2$, we have $\alpha_{t+1}\le \frac{\beta k}{4 L w^{1/3}}\leq 1$ for all $t$.
Then, we first consider $\eta_{t}^{-1}\|\epsilon_{t+1}\|^2 - \eta_{t-1}^{-1}\|\epsilon_t\|^2$. Using Lemma~\ref{lemma:epsilonrecursion}, we obtain
\begin{align*}
& \mathbb{E}\big[\eta_{t}^{-1}\|\epsilon_{t+1}\|^2 - \eta_{t-1}^{-1}\|\epsilon_t\|^2 \big] \\
\leq & \mathbb{E} \big[2 c^2 \eta_{t}^3 G^2 + \frac{(1-\alpha_{t+1})^2 (1+4 L^2 \eta_{t}^2)\|\epsilon_{t}\|^2}{\eta_{t}}
+4 (1-\alpha_{t+1})^2 L^2 \eta_{t}\|\nabla J_{\hat{\gamma}}(\theta_{t})\|^2 - \frac{\|\epsilon_t\|^2}{\eta_{t-1}}\big]\\
\leq & \mathbb{E} \bigg[\underbrace{2 c^2 \eta_{t}^3 G^2}_{A_t}+\underbrace{\big(\eta_{t}^{-1}(1-\alpha_{t+1})(1+4 L^2 \eta_{t}^2) - \eta_{t-1}^{-1}\big)\|\epsilon_t\|^2}_{B_t} + \underbrace{4 L^2 \eta_{t} \|\nabla J_{\hat{\gamma}}(\theta_{t})\|^2}_{C_t} \bigg] .
\end{align*}

Let start with upper bounding the second term $B_t$
we have
\begin{align*}
B_t \leq (\eta_{t}^{-1} - \eta_{t-1}^{-1}  +  \eta_{t}^{-1}(4L^2 \eta_{t}^2 - \alpha_{t+1}) )\|\epsilon_t\|^2
=\big(\eta_{t}^{-1} - \eta_{t-1}^{-1} + \eta_t(4L^2- \beta)\big)\|\epsilon_t\|^2~.
\end{align*}
Let us focus on $\frac{1}{\eta_t} - \frac{1}{\eta_{t-1}}$ for a minute. Using the concavity of $x^{1/3}$, we have $(x+y)^{1/3}\le x^{1/3} + yx^{-2/3}/3$. Therefore:
\begin{equation}
    \begin{aligned}
        \frac{1}{\eta_t} - \frac{1}{\eta_{t-1}}& = \frac{1}{k}\Big (w+\sum_{i=1}^tG_i^2\Big )^{1/3} - \frac{1}{k}\Big(w+\sum_{i=1}^{t-1}G_i^2\Big)^{1/3} \leq \frac{G_t^2}{3k(w+\sum_{i=1}^{t-1}G_i^2)^{2/3}} \\
        & \leq \frac{G_t^2}{3k(w-G^2+\sum_{i=1}^{t}G_i^2)^{2/3}} \leq \frac{G_t^2}{3k(w/2+\sum_{i=1}^{t}G_i^2)^{2/3}} \\
        & \leq \frac{2^{2/3}G_t^2}{3k(w+\sum_{i=1}^{t}G_i^2)^{2/3}} \leq \frac{2^{2/3}G_t^2}{3k^3}\eta_t^2\leq \frac{2^{2/3}G^2}{12Lk^3}\eta_t\leq \frac{G^2}{7Lk^3}\eta_t 
    \end{aligned}
\end{equation}
where we have used that that $w\geq (4Lk)^3$ to have $\eta_{t}\leq \frac{1}{4L}$.

Further, since $\beta =28 L^2 + G^2/(7 L k^3)$, we have
\begin{align*}
\eta_t(4L^2-\beta) 
\leq - 24 L^2 \eta_t  - G^2 \eta_t /(7 L k^3) .
\end{align*}
Thus, we obtain 
\begin{equation}
\begin{aligned}
B_t \leq  - 24 L^2 \eta_t\|\epsilon_t\|^2
\end{aligned}
\end{equation}

Now, we are ready to analyze the potential $\Phi_t$. Since $\eta_{t}\leq \frac{1}{4L}$, we can use Lemma~\ref{lemma:obj} to obtain
\begin{align*}
\mathbb{E}[\Phi_{t}-\Phi_{t+1}]
&\leq \mathbb{E} \left[- \frac{\eta_t}{4} \|\nabla J_{\hat{\gamma}}(\theta_t)\|^2 + \frac{3\eta_t}{4}\|\epsilon_t\|^2 + \frac{1}{32L^2 \eta_{t}}\|\epsilon_{t+1}\|^2 - \frac{1}{32L^2 \eta_{t-1}}\|\epsilon_t\|^2\right]~.
\end{align*}
Summing over $t$, we obtain
Rearranging terms we get,
\begin{equation}
\begin{aligned}
\E[\frac{\eta_t}{8}\|\nabla J_{\hat{\gamma}}(\theta_t)\|^2]&\leq \E[\Phi_{t+1}-\Phi_t] + \E[\frac{\beta^2\eta_t^3G^2}{16L^2}]\\
\Longleftrightarrow
\E[\frac{1}{8\eta_t^2}\|\nabla J_{\hat{\gamma}}(\theta_t)\|^2] &\leq \E[\frac{1}{8\eta_t^3}[ \Phi_{t+1}-\Phi_t]] +\frac{\beta^2G^2}{16L^2}\\
\end{aligned}
\end{equation}
Summing over $1,\cdots, t$, we have
\begin{equation}
    \begin{aligned}
    \sum\limits_{t=1}^T\E[\frac{1}{\eta_t^2}\|\nabla J_{\hat{\gamma}}(\theta_t)\|^2]&\leq \sum\limits_{t=1}^T\E[\frac{8}{\eta_t^3}[\Phi_{t+1}-\Phi_t]]+\frac{G^2T}{2L^2}\\
    \Longleftrightarrow 
    \sum\limits_{t=1}^T\E[\frac{1}{\eta_t^2}\| \nabla J_{\hat{\gamma}}(\theta_t)\|^2]&\leq \sum\limits_{t=1}^T\E[\frac{8}{\eta_t^3}[\Phi_{t+1}-\Phi_t]]+\frac{\beta^2G^2T}{2L^2}\\
    \Longleftrightarrow 
    \sum\limits_{t=1}^T\mathcal{W}_{1t} \E [\|\nabla J_{\hat{\gamma}}(\theta_t)\|^2] & \leq \sum\limits_{t=1}^T 8\mathcal{W}_{2t} \E[\Phi_{t+1}-\Phi_t] + \frac{\beta^2G^2T}{2L^2}\\
    \end{aligned}
\end{equation}
As $G_{t+1}^2 \leq G^2$, therefore $\eta_t \sim \Omega ((\frac{k}{w+tG^2})^{1/3})$. As a result, 
 $\mathcal{W}_{1t} = \frac{1}{\eta_t^2} = \frac{(w+t G^2)^{2/3}}{k^2}\sim O(t^{2/3})$, $\mathcal{W}_{2t} = \frac{1}{\eta_t^3} = \frac{(w+t G^2)}{k^3}\sim O(t)$.
\begin{equation}
    \begin{aligned}
       \sum\limits_{t=1}^T t\E[\Phi_{t+1} -\Phi_t]  &= \sum\limits_{t=1}^T \E[(t + 1)\Phi_{t+1} - (t)\Phi_{t}] -\sum\limits_{t=1}^{T}\Phi_{t+1}\\
       &= (T+1)\Phi_{T+1} - \Phi_{1} - \sum\limits_{t=1}^T \Phi_{t+1} = \sum\limits_{t=1}^{T+1} (\Phi_{T+1} -\Phi_{t}) \leq (T+1) \Delta_{\Phi}
    \end{aligned}
\end{equation}
where $\Delta_{\Phi} \leq \Delta_{J_{\hat{\gamma}}} + \frac{\|\epsilon_0\|^2}{32\eta_0L^2},\Delta_{J_{\hat{\gamma}}} =   J_{\hat{\gamma}}(\theta^*)-J_{\hat{\gamma}}(\theta), \forall \theta\in R^d$, and $\theta^\star$ is the maximizer of $J_{\hat{\gamma}}$.
\begin{equation}
\begin{aligned}
\sum\limits_{t=1}^T \mathcal{W}_{1t} = \sum\limits_{t=1}^T t^{2/3}\geq \int_{t=1}^T t^{2/3}dt = \frac{3}{5}(T^{5/3} - 1)\geq \frac{2}{5}T^{5/3}.
\end{aligned}
\end{equation}
Then we have
\begin{equation}
    \begin{aligned}
      \frac{\sum\limits_{t=1}^T\mathcal{W}_{1t}\E [\|\nabla J_{\hat{\gamma}}(\theta_t)\|^2}{\sum\limits_{t=1}^T \mathcal{W}_{1t}}  & \leq \frac{\sum\limits_{t=1}^T 8\mathcal{W}_{2t} \E[\Phi_t -\Phi_{t+1}]}{\sum\limits_{t=1}^T \mathcal{W}_{1t}} + \frac{\beta^2G^2T}{2L^2\sum\limits_{t=1}^T \mathcal{W}_{1t}}\\
      & \leq \frac{8(T+1)\Delta_{\Phi}}{\frac{2}{5}(T^{5/3} )} + \frac{\eta^2G^2T}{2L^2(\frac{2}{5}T^{5/3})}\\
      & \leq \frac{40\Delta_{\Phi}}{T^{2/3}} + \frac{2\beta^2G^2}{L^2T^{2/3}}
    \end{aligned}
\end{equation}
where  $\beta=28L^2 + \sigma^2/(7 L k^3)$.

\end{proof}

\section{Details of Experiments}

For VOMPS and ACE-STORM, the policy function $\pi$ is parameterized as a diagonal Gaussian distribution where the mean is the output of a two-hidden-layer network (64 hidden units with ReLU) and the standard deviation is fixed. For GeoffPAC, ACE, SVRPG, SRVR-PG, DDPG and TD3, we use the same parameterization as~\cite{geoffpac},~\cite{svrpg},~\cite{sarahpg},~\cite{lillicrap2015continuous} and~\cite{fujimoto2018addressing} respectively.

\paragraph{Cartpole}
\texttt{CartPoleContinuous-v0} has 4 dimensions for a state and 1 dimension for an action. The only difference between \texttt{CartPoleContinuous-v0} and \texttt{CartPole-v0} (provided by OpenAI Gym) is that \texttt{CartPoleContinuous-v0} has a continuous value range of $[-1,1]$ for action space. The episodic return for the comparison with on-policy and off-policy methods is shown in Fig.~\ref{fig:episodic:cartpole:onpol},~\ref{fig:episodic:cartpole:offpol}. The relative performance matches with that of the Monte Carlo return. 
\paragraph{Hopper} \texttt{Hopper-v2} attempts to make a 2D robot hop that has 11 dimensions for a state and 3 dimensions for an action. The episodic return for the comparison with on-policy and off-policy methods is shown in Fig.~\ref{fig:episodic:hopper:onpol},~\ref{fig:episodic:hopper}.
\paragraph{HalfCheetah}
\texttt{HalfCheetah-v2} attempts to make a 2D cheetah robot run that has 17 dimensions for a state and 6 dimensions for an action. The episodic return for the comparison with on-policy and off-policy methods is shown in Fig.~\ref{fig:episodic:halfcheetah:onpol},~\ref{fig:episodic:halfcheetah}.

\begin{figure*}[htb!]
\centering
\begin{subfigure}{.44\textwidth}
  \centering
  \includegraphics[height=3.5cm,width=1.\textwidth]{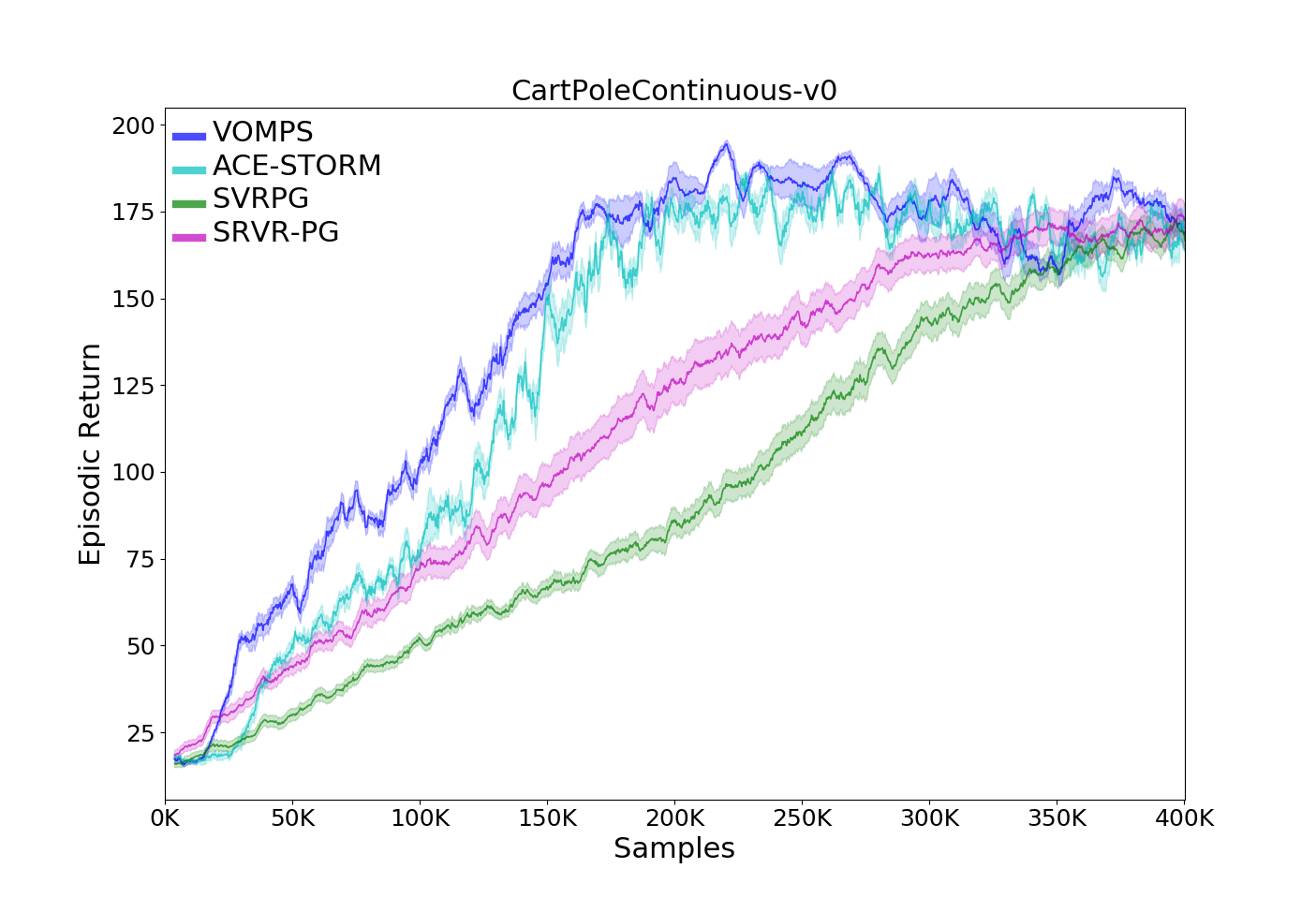}
  \caption{Comparison with on-policy methods}
  \label{fig:episodic:cartpole:onpol}
\end{subfigure}
\begin{subfigure}{.44\textwidth}
  \centering
  \includegraphics[height=3.5cm,width=1.\textwidth]{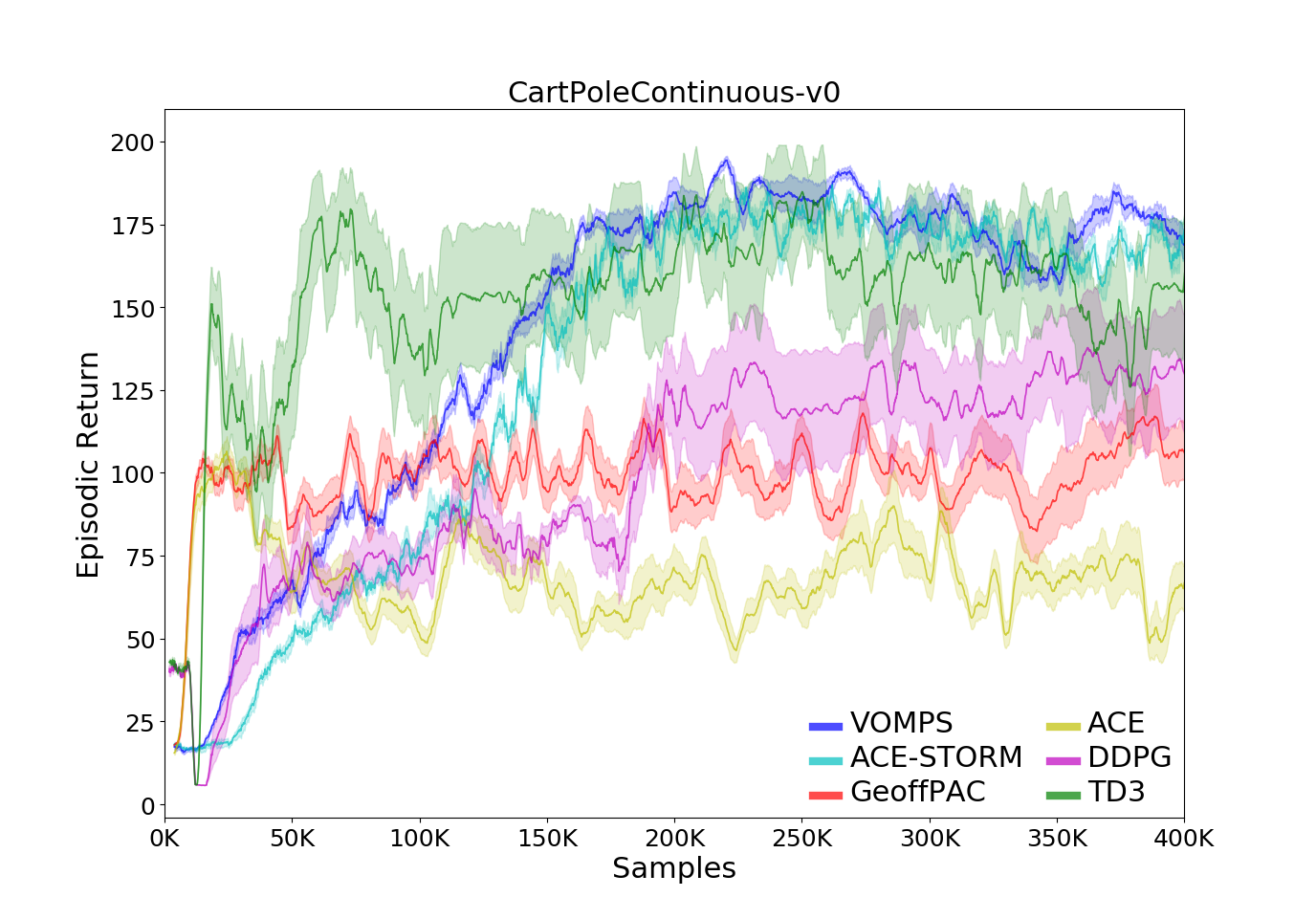}
  \caption{Comparison with off-policy methods}
  \label{fig:episodic:cartpole:offpol}
\end{subfigure}
\caption{Episodic Return on \texttt{CartPoleContinuous-v0}}
\end{figure*}

Besides, the episodic return for the $20\%$ action noise comparison on Mujoco (including \texttt{Hopper-v2} and \texttt{HalfCheetah-v2}) is shown in Fig.~\ref{fig:episodic:hopper:noise:onpol},~\ref{fig:episodic:hopper:noise:offpol},~\ref{fig:episodic:halfcheetah:noise:onpol},~\ref{fig:episodic:halfcheetah:noise:offpol} respectively.

It should be noted that the parameter settings for GeoffPAC and ACE are insensitive on \texttt{CartPoleContinuous-v0}. Therefore, we keep the setting of $\lambda^{(1)}=0.7$, $\lambda^{(2)}=0.6$, $\hat{\gamma}=0.2$ for GeoffPAC, and $\lambda^{(1)}=0$ for ACE in all of the experiments. For DDPG and TD3, we use the same parameter settings as \cite{lillicrap2015continuous} and~\cite{fujimoto2018addressing} respectively.

\begin{footnotesize}
\begin{figure*}[htb!]
\begin{subfigure}{.24\textwidth}
  \centering
  \includegraphics[width = 1\textwidth]{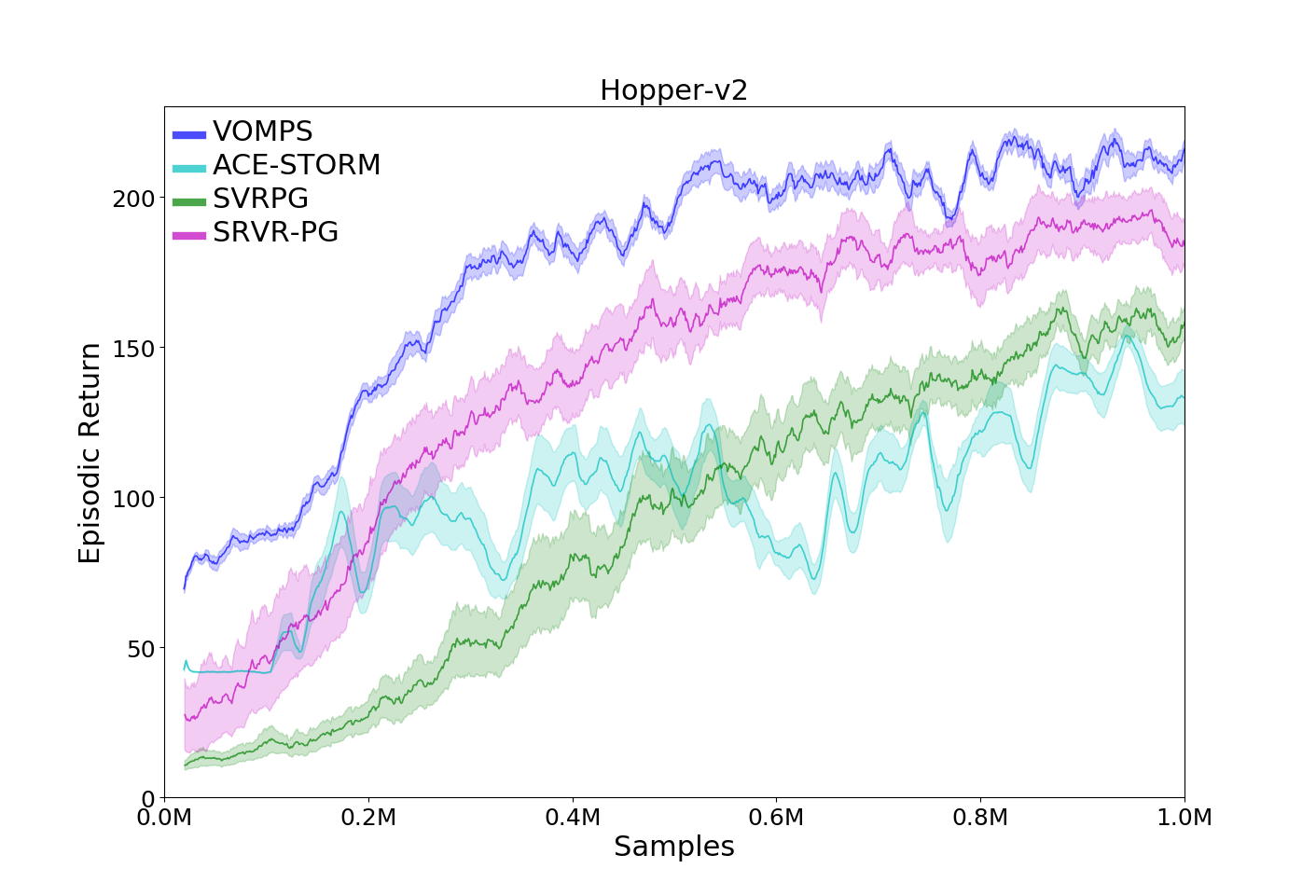}
  \caption{Hopper}
  \label{fig:episodic:hopper:onpol}
\end{subfigure}
\begin{subfigure}{.24\textwidth}
  \centering
  \includegraphics[width = 1\textwidth]{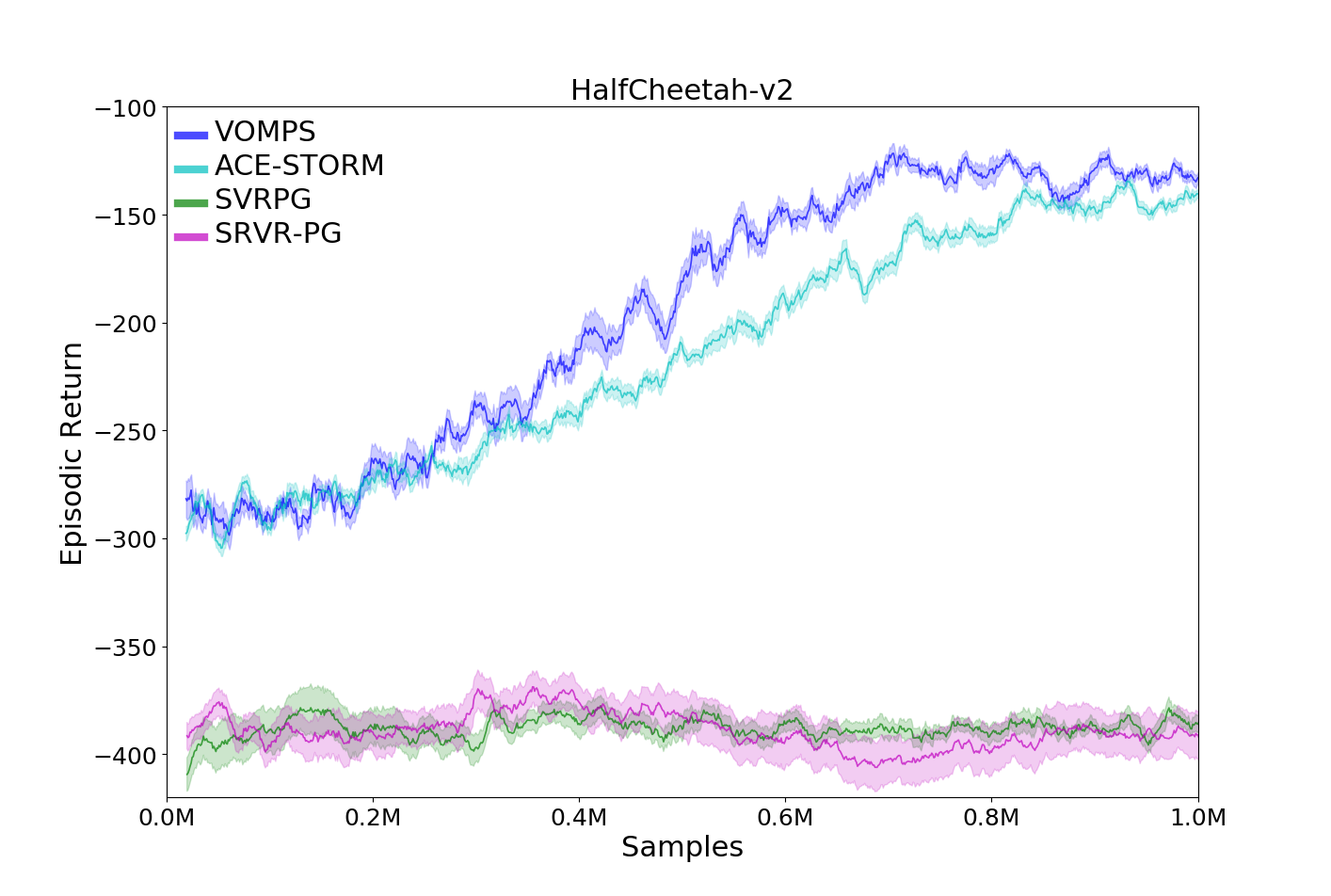}
    \caption{HalfCheetah}
    \label{fig:episodic:halfcheetah:onpol}
\end{subfigure}
\begin{subfigure}{.24\textwidth}
  \centering
  \includegraphics[width = 1\textwidth]{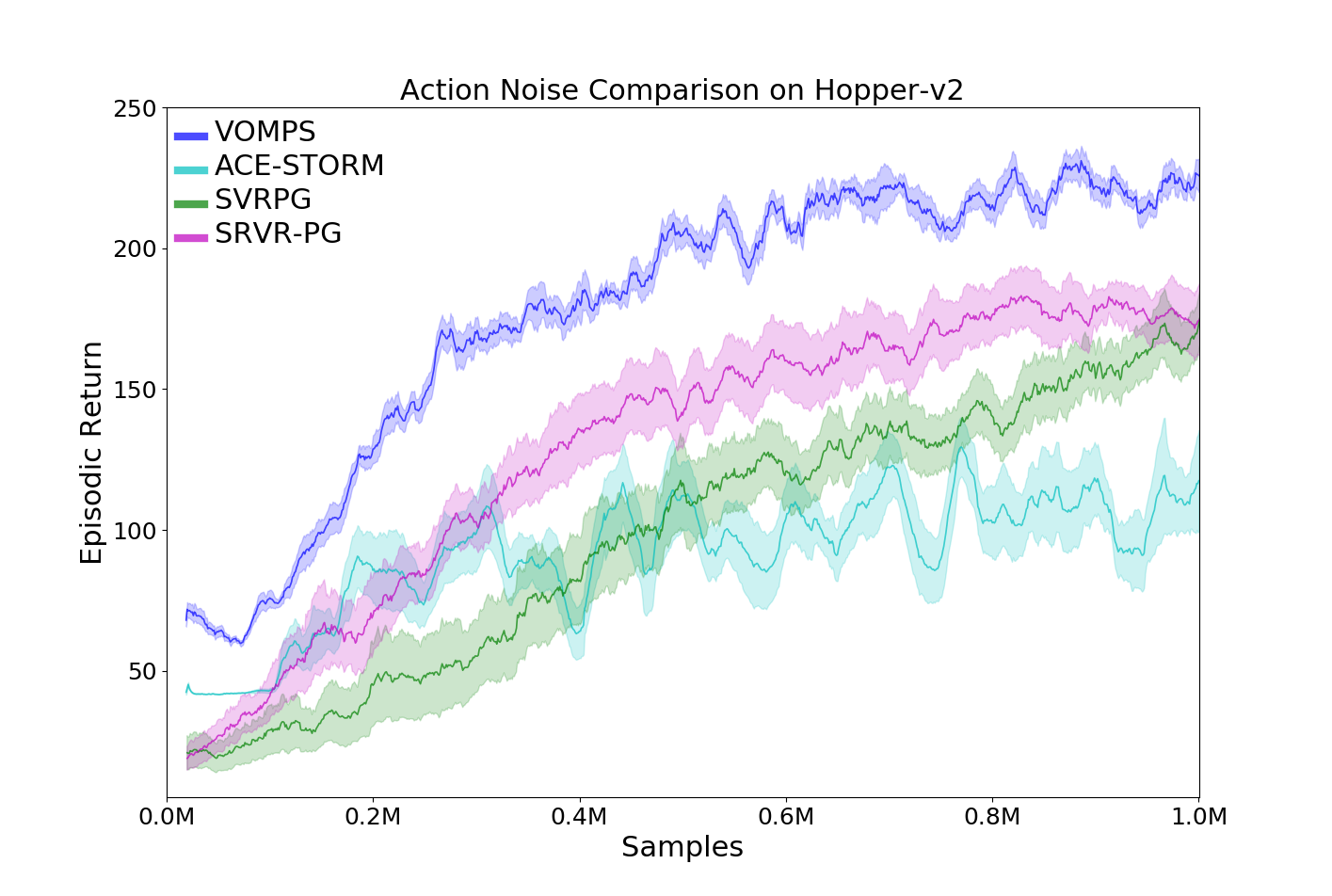}
  \caption{Hopper (action noise)}
  \label{fig:episodic:hopper:noise:onpol}
\end{subfigure}
\begin{subfigure}{.24\textwidth}
  \centering
  \includegraphics[width = 1\textwidth]{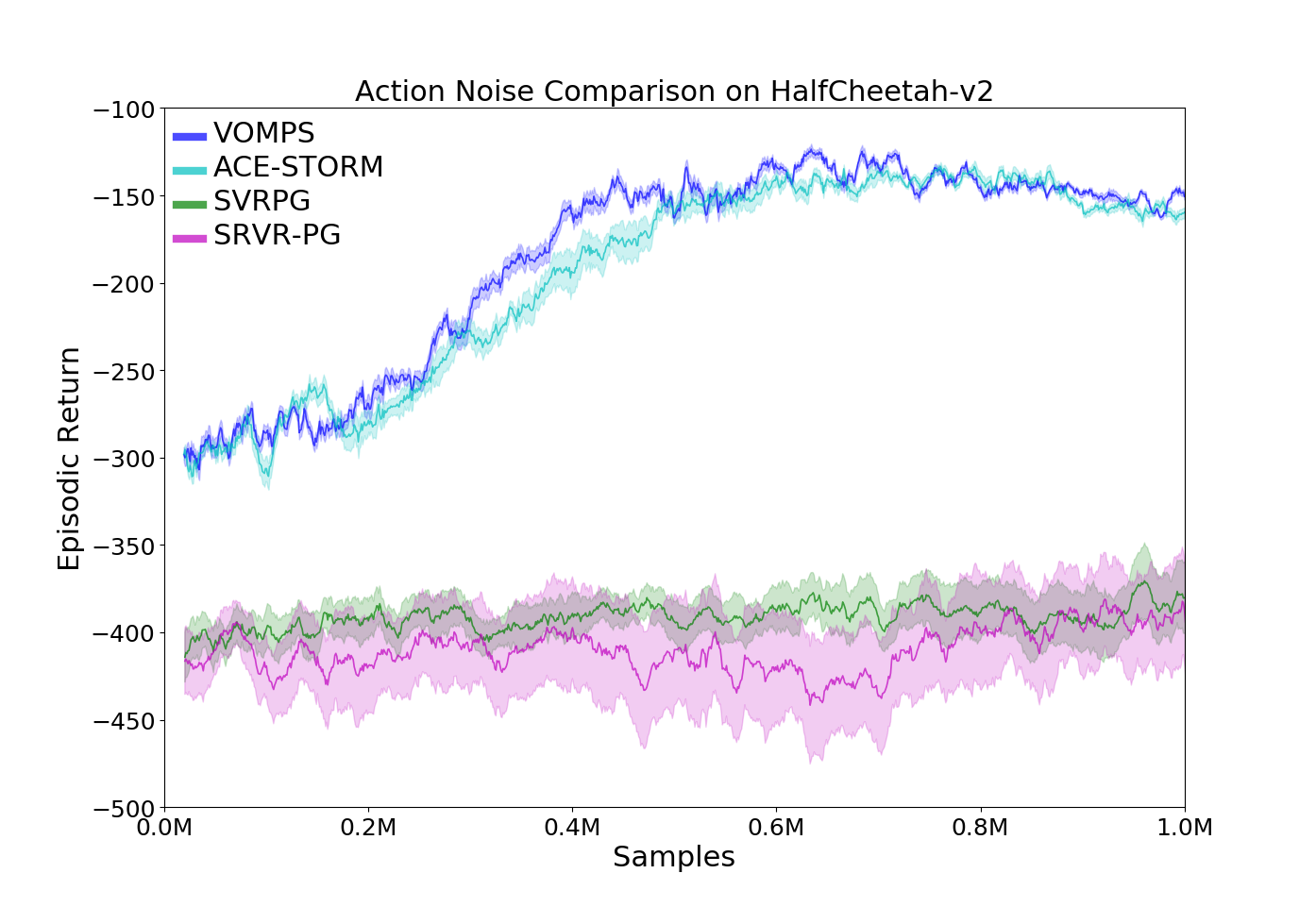}
  \caption{{\footnotesize HC (action noise)}}
  \label{fig:episodic:halfcheetah:noise:onpol}
\end{subfigure}
\caption{Comparison with on-policy PG methods (Mujoco), ``HC'' is short for HalfCheetah.}
\end{figure*}

\begin{figure*}[htb!]
\begin{subfigure}{.24\textwidth}
  \centering
  \includegraphics[width = 1\textwidth]{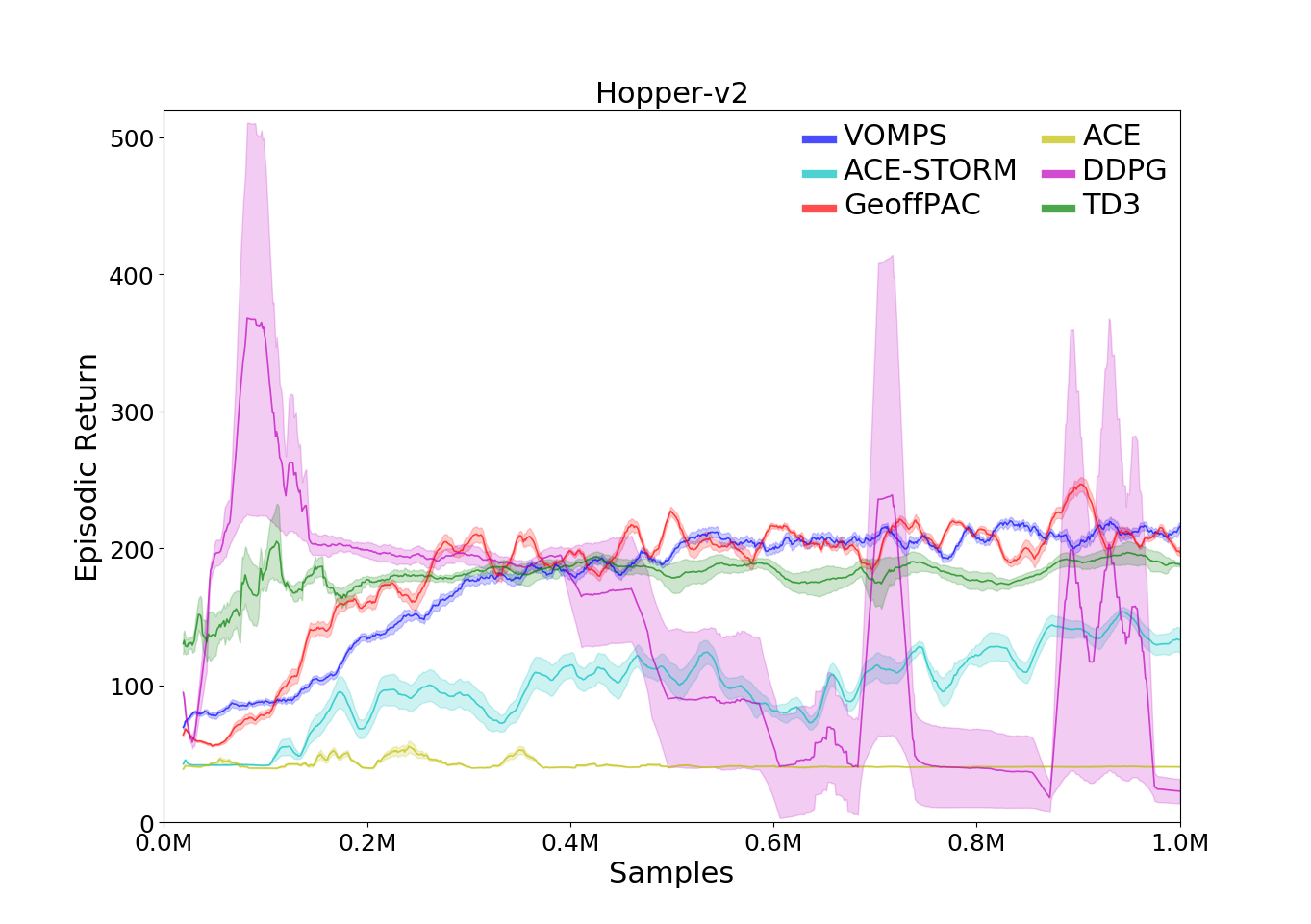}
  \caption{Hopper}
  \label{fig:episodic:hopper}
\end{subfigure}
\begin{subfigure}{.24\textwidth}
  \centering
  \includegraphics[width = 1\textwidth]{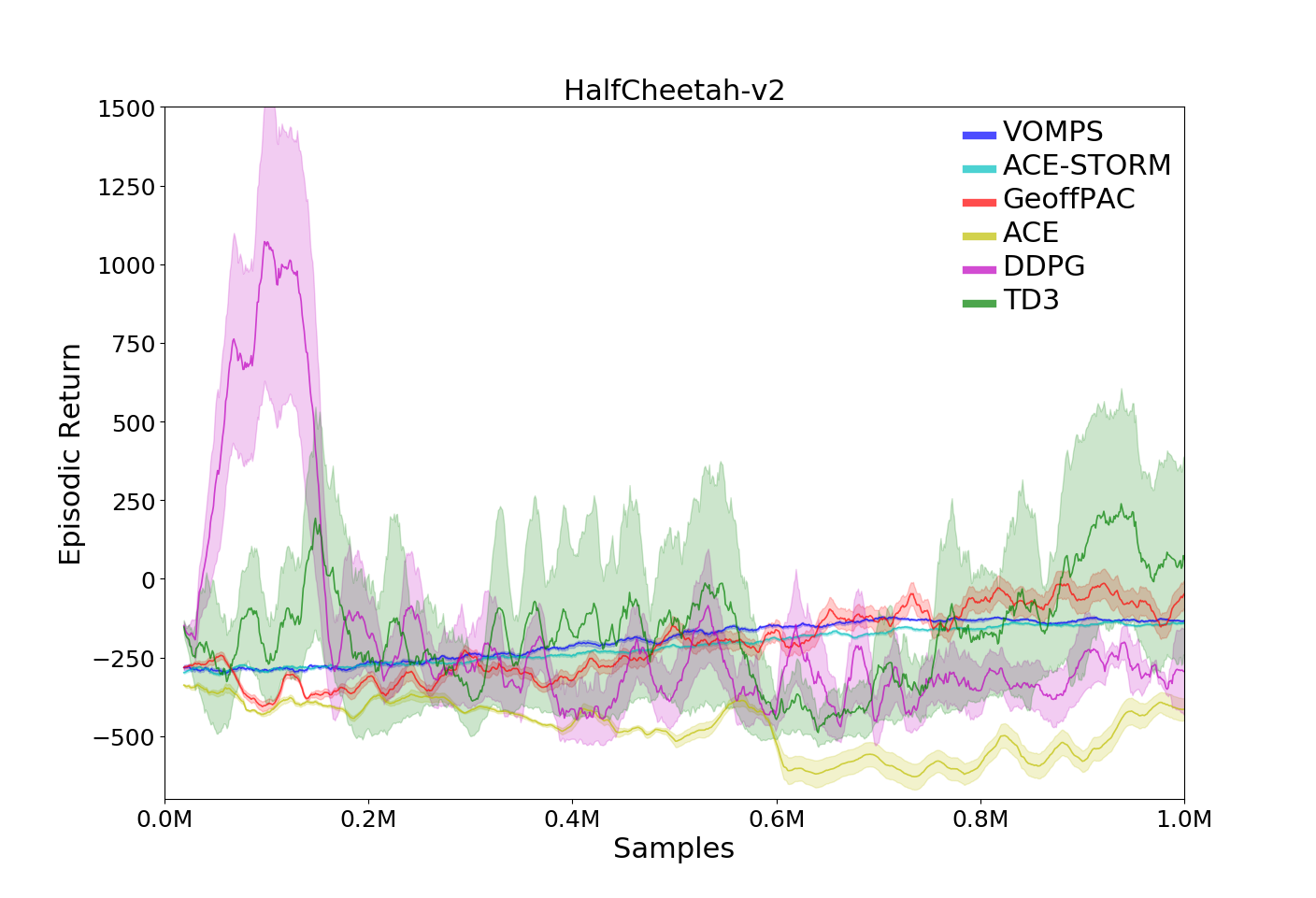}
    \caption{HalfCheetah}
    \label{fig:episodic:halfcheetah}
\end{subfigure}
\begin{subfigure}{.24\textwidth}
  \centering
  \includegraphics[width = 1\textwidth]{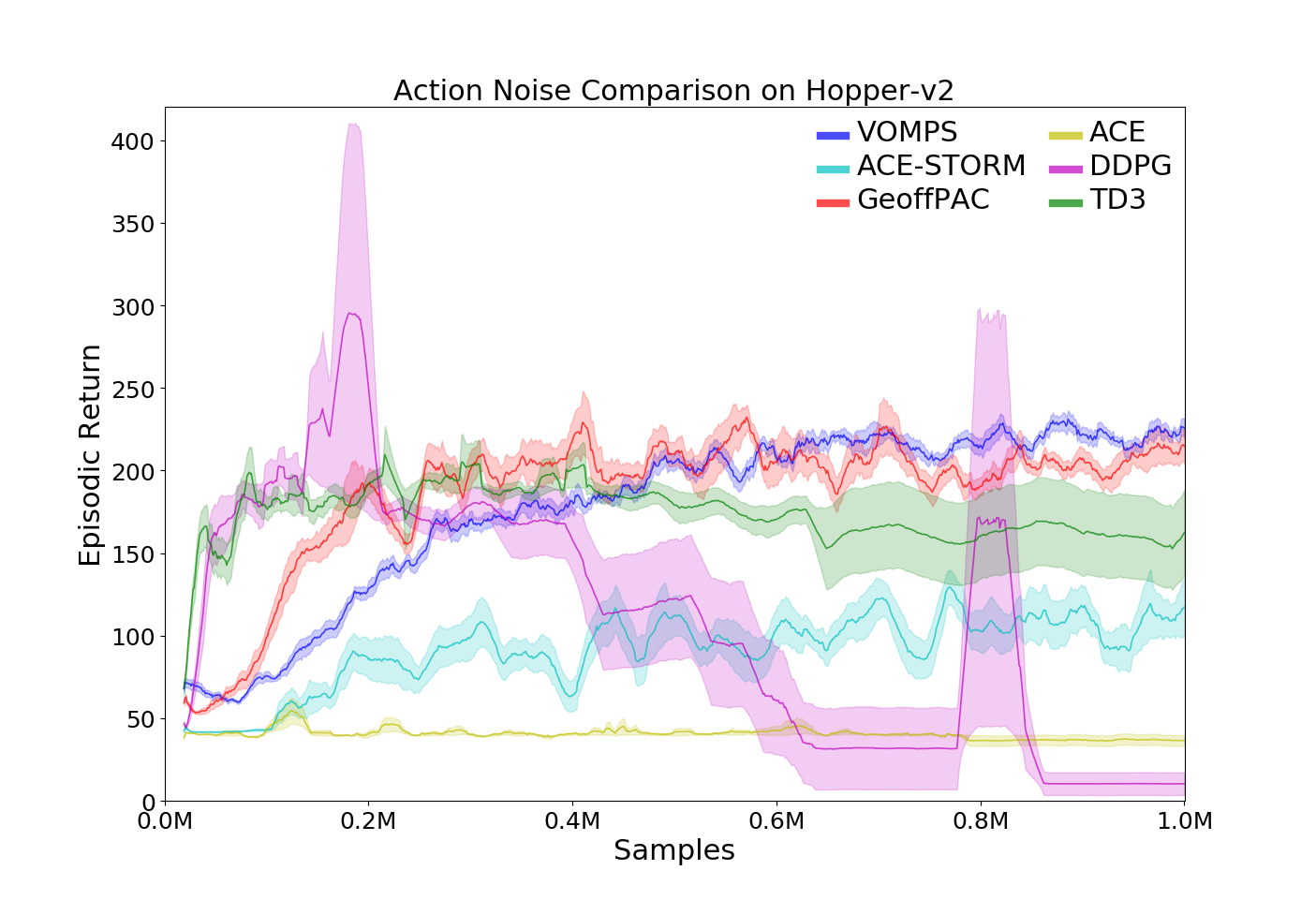}
  \caption{Hopper (action noise)}
  \label{fig:episodic:hopper:noise:offpol}
\end{subfigure}
\begin{subfigure}{.24\textwidth}
  \centering
  \includegraphics[width = 1\textwidth]{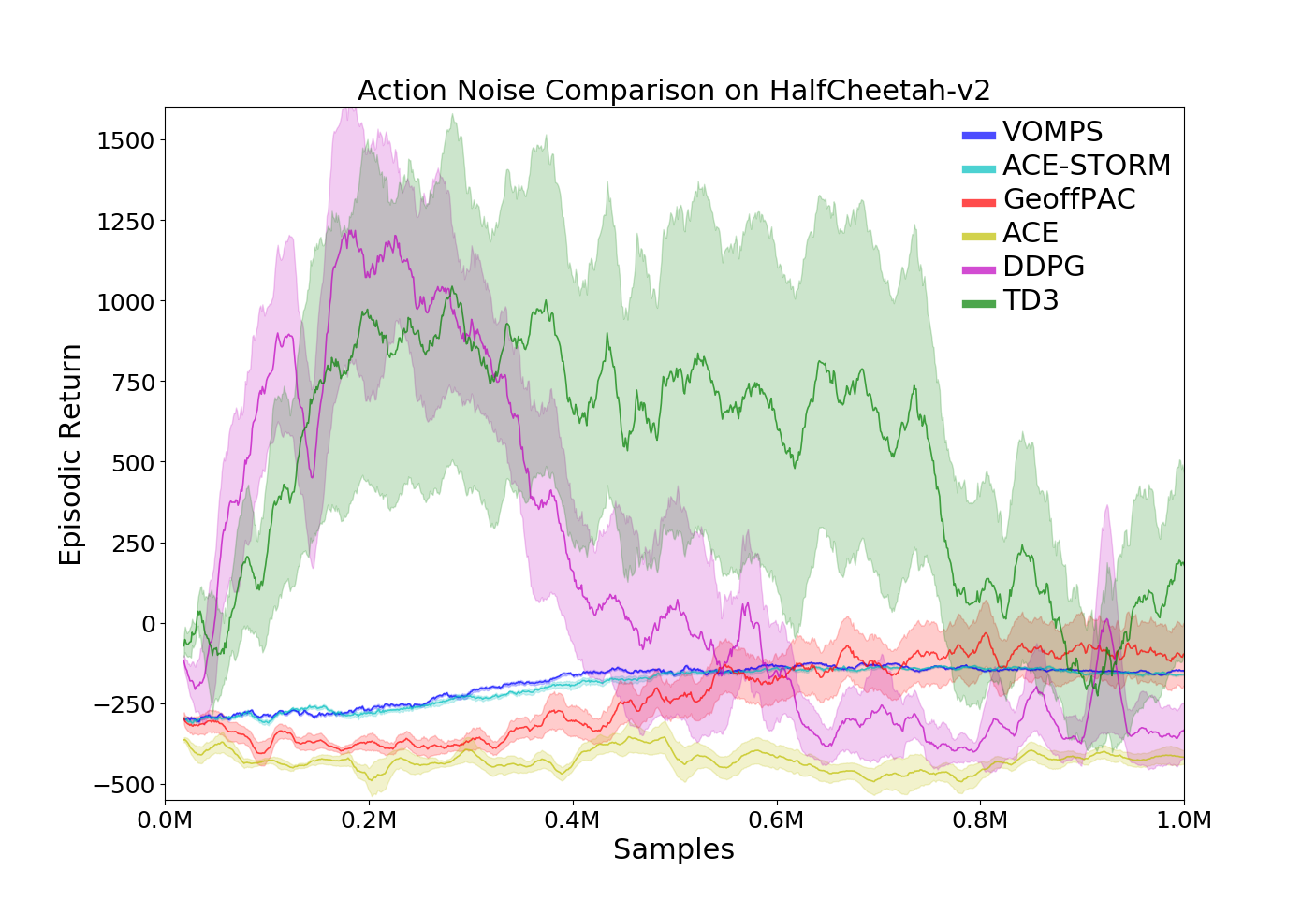}
    \caption{{\footnotesize HC (action noise)}}
  \label{fig:episodic:halfcheetah:noise:offpol}
\end{subfigure}
\caption{Comparison with off-policy PG methods (Mujoco), ``HC'' is short for HalfCheetah.}
\end{figure*}
\end{footnotesize}

\end{document}